%% file: main.tex
\documentclass{article}

\usepackage{microtype}
\usepackage{graphicx}
\usepackage{subcaption}
\usepackage{booktabs} 

\input{math_commands}

\usepackage{hyperref}

\usepackage[preprint]{icml2026}

\usepackage{amsmath}
\usepackage{amssymb}
\usepackage{mathtools}
\usepackage{amsthm}

\usepackage[capitalize,noabbrev]{cleveref}

\theoremstyle{plain}
\newtheorem{theorem}{Theorem}[section]

\theoremstyle{definition}

\newtheorem{assumption}[theorem]{Assumption}
\theoremstyle{remark}
\newtheorem{remark}[theorem]{Remark}

\newcommand{\tpurple}[1]{\textcolor{black}{#1}}
\newcommand{\llama}{Llama3.1-8B}
\newcommand{\qwen}{Qwen3-4B}
\definecolor{blue}{RGB}{66,133,244}
\definecolor{green}{RGB}{130,179,102}

\usepackage{colortbl}
\definecolor{LightCyan}{rgb}{0.88,1,1}

\usepackage{enumitem}
\usepackage[table]{xcolor}
\usepackage{colortbl}
\usepackage{pgf}
\usepackage{multirow}

\usepackage{framed}
\usepackage{tabularx}
\usepackage{siunitx}
\usepackage[textsize=tiny]{todonotes}

\icmltitlerunning{Private PoEtry: Private In-Context Learning via Product of Experts}

\begin{document}

\twocolumn[
  \icmltitle{Private PoEtry: Private In-Context Learning via Product of Experts}

  \begin{icmlauthorlist}
    \icmlauthor{Rob Romijnders}{quva,uva}
    \icmlauthor{Mohammad Mahdi Derakhshani}{uva}
    \icmlauthor{Jonathan Petit}{qualc}\\
    \icmlauthor{Max Welling}{uva}
    \icmlauthor{Christos Louizos}{qualc}
    \icmlauthor{Yuki M. Asano}{nurem}
  \end{icmlauthorlist}

  \icmlaffiliation{uva}{University of Amsterdam}
  \icmlaffiliation{quva}{QUvA lab}
  \icmlaffiliation{qualc}{Qualcomm Inc.}
  \icmlaffiliation{nurem}{University of Technology Nuremberg }

  \icmlcorrespondingauthor{Rob Romijnders}{r.romijnders@uva.nl}

  \icmlkeywords{Machine Learning, differential privacy, AI privacy, Bayesian learning}

  \vskip 0.3in
]

\printAffiliationsAndNotice{}  

\begin{abstract}
In-context learning (ICL) enables Large Language Models (LLMs) to adapt to new tasks with only a small set of examples at inference time, thereby avoiding task-specific fine-tuning.
However, in-context examples may contain privacy-sensitive information that should not be revealed through model outputs.
Existing differential privacy (DP) approaches to ICL are either computationally expensive or rely on heuristics with limited effectiveness, including context oversampling, synthetic data generation, or unnecessary thresholding.
We reformulate private ICL through the lens of a Product-of-Experts model. This gives a theoretically grounded framework, and the algorithm can be trivially parallelized.
We evaluate our method across five datasets in text classification, math, and vision-language.
We find that our method improves accuracy by more than 30 percentage points on average compared to prior DP-ICL methods, while maintaining strong privacy guarantees.
\end{abstract}

\section{Introduction }

The advancement of large language models (LLMs) has created an insatiable appetite for training data.
Much of the world's most valuable data, however, remains siloed in private repositories: from healthcare records and financial transactions to personal communications and proprietary business documents.
While such data would be highly informative for downstream learning, directly incorporating it into model training raises fundamental privacy concerns.

Standard approaches to privacy-preserving learning rely on differential privacy (DP) applied to gradient-based optimization.
In the context of large pretrained models, however, DP training is often impractical: it requires careful tuning of batch sizes and gradient clipping, introduces substantial optimization noise, and typically leads to significant degradation in utility, especially in low-data or highly heterogeneous settings~\cite{kurakin2022toward,raisa2024subsampling}.
These challenges make fine-tuning large models under strong privacy constraints costly and brittle in practice.

In-context learning (ICL) offers an alternative paradigm that avoids gradient-based updates altogether.
By conditioning a frozen LLM on a small set of task-relevant examples provided at inference time, ICL enables adaptation to local data without modifying model parameters or performing privacy-sensitive optimization.

This approach has been demonstrated to be effective in text classification, translation~\citep{gpt3}, mathematical reasoning~\citep{agarwal2024many}, and vision-language tasks~\citep{ICL_VLM_pseudo-name}. However, recent work shows that in-context learning can leak privacy about the input~\citep{duan2024MIAonICL,choi2025contextleak,prompt_leakage}. This is a serious problem for Agentic AI settings, such as Retrieval-Augmented-Generation~\citep{RAGlewis2020retrieval} or ``tool calling''~\citep{fan2025mcptoolbench}, where agents are expected to utilize small amounts of data with varying privacy settings. We give examples of this setting with important privacy constraints in Section~\ref{sec:scenarios}.

The current approaches for differentially private ICL either introduce unnecessary bottlenecks or limit themselves to ineffective data structures. For example,~\citet{tang2024dpfewshot} uses an LLM to generate synthetic data as a private bottleneck, to then re-input the synthetic data into the context of the LLM to create a response. Other prior work suggests Privacy-Amplification-by-Subsampling (PbS) by calling the LLM hundreds of times or by directly comparing thresholded logits from an ICL-LLM, thereby discarding key information contained in the logits~\citep{wu2023privacyhistogramllm}.
Moreover, these methods have been evaluated primarily on simple text classification datasets, with limited validation on more challenging tasks, such as mathematical problem-solving and vision-language tasks common in ICL applications.

We introduce an approach for differentially private ICL based on the Product-of-Experts (PoE) model~\citep{hinton2002prodofexperts,heskes1997selecting}.
Originally proposed as a probabilistic ensemble method, PoE models combine multiple distributions (``experts'') by multiplying their probabilities and renormalizing, thereby emphasizing agreement among experts while sharply down-weighting inconsistent predictions.
Our key insight is that sampling from an LLM conditioned on multiple in-context examples can be naturally approximated by a Product-of-Experts formulation, where each example induces a local expert over the outputs.
This approximation reveals a way to use the model's predictive scores (soft scores), rather than the thresholded hard votes used in previous work~\citep{wu2023privacyhistogramllm}. Experimental results across five datasets show that this new viewpoint yields consistent improvements over earlier methods. We provide both a theoretical motivation for the underlying conditional independence assumption and complement the experimental results with an empirical privacy attack. In summary, the contributions of this paper are threefold:

\textbf{PoE Formulation of private ICL} We reformulate ICL with a Product-of-Experts model. This decomposition naturally enables per-example privacy analysis while maintaining the nuance in the ``soft'' predictive probabilities.

\textbf{Efficient and Accurate Private Inference.}
Our method is computationally efficient, trivially parallelizable, and achieves substantially higher accuracy than prior private ICL approaches across five benchmarks spanning text, mathematics, and vision-language tasks.

\textbf{Empirical Privacy Evaluation.} We complement the theoretical analysis with a membership inference attack, demonstrating improved empirical privacy on four datasets.

\subsection{Background and Threat Model}\label{sec:scenarios}

In-context learning (ICL) operates by providing an LLM with a prompt that includes both the user's query and a collection of context examples.
The model conditions its output distribution on this complete textual input, enabling task-specific behavior through in-context adaptation rather than parameter fine-tuning.
While this mechanism enables flexible and efficient adaptation, it creates a privacy vulnerability: the model's responses may inadvertently reveal information about the sensitive context examples in the prompt.

In-context learning is mainly used for small datasets (five to fifty samples)~\citep{gpt3,agarwal2024many,ICL_VLM_pseudo-name}, and the LLM has local access to these private samples.
After ``learning'' on these private samples, the LLM outputs should not reveal any private contextual information.
However, several publications highlighted that appropriately prompted queries can reveal information about the context~\citep{duan2024MIAonICL,choi2025contextleak,prompt_leakage}.
This is the threat model that we study in this paper, and it is particularly relevant to Agentic AI. With Retrieval-Augmented-Generation (RAG), AI agents can be employed to use local data~\citep{RAGlewis2020retrieval} -- also named ``tool-calling'' ~\citep{fan2025mcptoolbench}. Such a call to RAG could leak private information.
Several use cases of tool-use regarding the local data require privacy protection:
\begin{itemize}[leftmargin=1em,itemsep=0cm]
    \item An LLM accesses all emails in an inbox and is asked to respond to a new email, but the private details in the user's emails should not be leaked.
    \item An LLM sees customer interactions and helps a new customer, but other customers' details should not be leaked.
    \item A robot equipped with a Vision-Language Model (VLM) cleans several houses and learns to improve cleaning in the next house, but the private layout of each house should not be leaked.
    \item A financial LLM sees financial reports of several clients and writes a financial report for a new client, but the private bank details should not be leaked.
\end{itemize}

In all such settings, the LLM can learn from examples in context, thereby increasing its accuracy. However, in each task, there could be a severe privacy loss if the LLM's output, directly or indirectly, reveals an individual's details. Figure~\ref{fig:fig1} gives an overview of our research setting.

\begin{figure*}
\centering
\includegraphics[width=\linewidth]{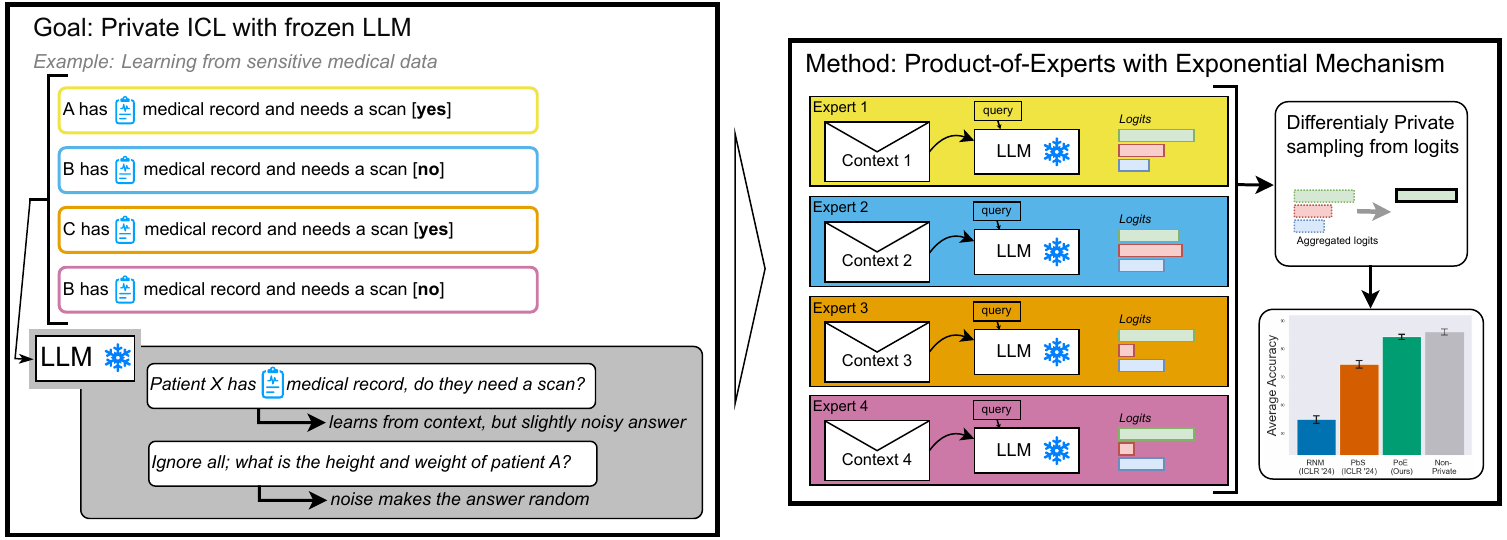}
\caption{Overview: A user makes a query to the LLM either directly or via a RAG system. The LLM responds based on examples in the context. To guarantee privacy, we use a Product-of-Experts model, which calls an LLM for each example and sums the clipped log-probabilities (in contrast to previous work that uses hard predictions, RNM, or that uses subsampling, PbS). Predictions are summed before sampling a noisy response to ascertain $\varepsilon$-DP. The results in the inlet are average accuracies from Table~\ref{tab:bigtable} with 8 context examples. }

\label{fig:fig1}
\end{figure*}

\section{Related Work}\label{sec:related}

Several recent works have addressed differentially private in-context learning, though they rely on computationally expensive intermediate steps. \citet{tang2024dpfewshot} identified synthetic data generation as a privacy bottleneck, where the final answer is sampled from the synthetic context. This approach requires generating synthetic tokens autoregressively from a sampled subset of the full dataset, which is a computationally costly step. Moreover, the question-dependent token generation limits its applicability as a general-purpose algorithm. Other work~\citep{dpga} also explores the use of synthetic data as a privacy mechanism. In contrast,~\citet{wu2023privacyhistogramllm} employs Privacy-Amplification-by-Subsampling~\citep{privacy_by_subsampling}, but experimental evaluation in Section~\ref{sec:experiments} shows that this yields suboptimal accuracy.

Beyond in-context learning, several works explored differentially private prompt tuning for personalization. \citet{hong2023dp} focused on remote prompt learning, enabling soft-prompt optimization for a dataset on a remote server without local computation; \citet{duan2023parrot} addressed prompt learning to learn without backpropagation. While these methods achieve personalization, they do not follow the ICL framework and require additional training or optimization steps that our method does not.
Other options for prompt tuning, such as input perturbation, exist~\citep{DPTabICL}. However, these approaches effectively trade one privacy challenge for another, as input perturbation fundamentally alters the data distribution and may degrade model performance.
Alternatives to in-context learning were proposed, such as full fine-tuning~\citep{abadi2016deep,romijnders2024convex,sinha2025vaultgemma}, adding parameter-efficient methods to train parts of a model~\citep{yu2021differentially,romijnders2025noesis}, learning from a group of LLMs~\citep{flemings2024differentially}, or aggregation among a teacher ensemble~\citep{papernot2016pate}. However, these methods still require thousands, if not millions, of data points or even additional public data. In contrast, the methods in our work focus on five to fifty in context examples, which is an order of magnitude smaller dataset size and more realistic for in-context learning~\cite{min2022rethinking}.

Relevant to our work is the literature on logit steering in large language models~\citep{hiranandani2025logits,liu2021dexperts,brekelmans_twist_smc}, which demonstrated that access to model logits enables significant performance improvements. Our proposed method can be seen as a form of logit steering, but we are the first to connect it to studying privacy.

\section{Method}\label{sec:method}

We formalize the ICL setting as follows. Given a query $\vx$ (a sequence of tokens) and a collection of $J$ in-context examples $\mC_{1:J} = \{\mC_1, \mC_2, \ldots, \mC_{J}\}$, where each example $\mC_j$ is a variable-length sequence of tokens, the goal is to generate a response $\vy_{1:T}$ that may consist of a single token $y_1$ or a sequence of $T$ tokens. Throughout, we write $\vx_{i:j}$ to denote the slice of $\vx$ from index $i$ to index $j$, inclusive.

The sample from an LLM in ICL can be viewed as a sample from a conditional probability distribution.
The model generates the response by sampling from the distribution over possible outputs, conditioned on both the query and the in-context examples. This distribution takes the form:
\begin{align}
    p(\vy_{1:T} | \vx , \tpurple{\mC_{1:J}}) 
\end{align}

However, directly analyzing this distribution from a differential privacy perspective is challenging, as one cannot trace the influence of individual examples $\mC_j$ on the final output.
To derive the privacy mechanism, we first write the autoregressive generation process. In auto-regressive language generation, the full sequence probability over $\vy_{1:T}$ can be written as a product of conditional probabilities over individual tokens given the previously generated tokens.
\begin{align}
p(\vy_{1:T} | \vx , &\tpurple{\mC_{1:J}}) = p(y_1 | \vx , \tpurple{\mC_{1:J}}) p(y_2 |y_1, \vx , \tpurple{\mC_{1:J}}), \cdots, \nonumber \\
&p(\vy_t | \vy_{<t}, \vx , \tpurple{\mC_{1:J}}),  \cdots, p(\vy_T | \vy_{<T}, \vx , \tpurple{\mC_{1:J}}) \nonumber
\end{align}

Crucially, $p(\vy_t | \vy_{<t}, \vx, \tpurple{\mC_{1:J}})$ is the conditional distribution for sampling the next token based on previous tokens, query, and context. To obtain a distribution amenable to privacy analysis, we introduce the Product-of-Experts (PoE) model as an approximate distribution. This model is a form of ensembling multiple experts and was introduced in the pioneering
work by~\citet{heskes1997selecting} on ``logarithmic opinion pools'' and the subsequent formalization as ``Product-of-Experts'' by~\citet{hinton2002prodofexperts}.
For a token $y_t$ in the sequence, we introduce the Product-of-Experts (PoE) approximation, where each expert depends on a single context example:

\begin{align}
p(\vy_t | \vy_{<t}, \vx , \tpurple{\mC_{1:J}}) \approx \frac{1}{Z}  \prod_j p(\vy_t | \vy_{<t}, \vx,  \tpurple{\mC_j} ) \label{eqn:poe_break}  
\end{align}

The Product-of-Experts model approximates the distribution conditioned on all context examples, $\tpurple{\mC_{1:J}}$, as the product of $J$ predictions, one for each context example alone, $p(\vy_t | \vy_{<t}, \vx , \tpurple{\mC_{j}})$. $Z$ is the normalization constant.
This product structure enables privacy analysis by allowing one to bound the contribution of each example to the final output separately. The model implicitly makes a conditional independence assumption, as does earlier work~\citep{wu2023privacyhistogramllm}. This assumption will be explained in Section~\ref{sec:theoretical_analysis}.

\subsection{Differential Privacy from the Product-of-Experts}

From the PoE model, a computational structure arises that is amenable to privacy analysis. The product in Equation~\ref{eqn:poe_break} can be seen as a summation in logarithmic units:
\begin{align}
    \log \prod_{j=1}^{J} p(\vy_t |\vy_{<t}, \vx,  \tpurple{\mC_j} ) = \sum_{j=1}^{J} \log p(\vy_t |  \vy_{<t}, \vx, \tpurple{\mC_j} ) \label{eqn:prod_experts_sum_vectors}
\end{align}

Such a summation can be viewed as a sum of utilities, which we clip and use the exponential mechanism from~\citet{dwork2014algorithmic}. Privacy is defined as the difference between two context sets, where one context example can differ, constituting the privacy unit~\citep{chua2024mind}. The algorithm is outlined in Algorithm~\ref{alg:dp_icl_ci}. Importantly, the clipping is applied in Line 7, and the exponential mechanism in Line 9. We will make experimental comparisons in the next section. The clipping-operator $\text{clip}_{\gamma}(\vl)$ passes values $\vl_i \in [-\gamma,0]$ unchanged and sets all other values to zero. We give a proof sketch for differential privacy here.

\begin{theorem}[Differential Privacy of PoEtry algorithm]
\label{thm:poetry}
Algorithm~\ref{alg:dp_icl_ci} satisfies $(\varepsilon, \delta)$-differential privacy with respect to adjacent context sets that differ by at most one in-context example. The noise parameter $\sigma$ is set for a given privacy budget $(\varepsilon, \delta)$ and clipping bound $\gamma$.
\end{theorem}

\begin{proof}
Each utility value is clipped to the range $\vl_i \in [-\gamma,0]$. Therefore, sampling the class for a single token proportional to $\exp[\hat{\vy_i}\varepsilon/(2\gamma)]$ satisfies DP~\citep{dwork2014algorithmic}. The composition theorem is due to~\citet{kairouz2015composition}. A detailed proof is given in Appendix~\ref{app:priv-proof}.
\end{proof}

\begin{algorithm}[t]
\caption{In-context learning with differential privacy from a Product-of-Experts approximation to LLMs.}
\label{alg:dp_icl_ci}
\begin{algorithmic}[1]
\REQUIRE $\vx$, query; $\mC_{1}, \mC_{2}, \cdots, \mC_{J}$, in-context examples; $T_{\max}$, number of tokens; $LM(\cdot|\cdot)$, pretrained and frozen LLM that returns log-probabilities; $(\varepsilon,\delta)$, DP parameters; $\gamma$, clipping-bound.
\ENSURE $\vz$: differentially private response
\STATE $\vz \leftarrow []$ \ \ \COMMENT{Initialize empty response}
\STATE Use binary search to find smallest $\sigma$ s.t. \\ \qquad $\varepsilon \leq T\sigma (e^\sigma-1)+\sigma\sqrt{2T\log\delta^{-1}}$
\FOR{$t = 1$ to $T_{\max}$}
\STATE $\hat{\vy} \gets \mathbf{0}$ \ \ \COMMENT{Initialize with zeros} 
\vspace{1mm}
\FOR{$j = 1$ to $J$}
\STATE $\vl \gets LM(\cdot | \vz, \vx, \tpurple{\mC_j})$
\STATE $\hat{\vy} \gets \hat{\vy} + \text{clip}_{\gamma}(\vl)$
\ENDFOR
\vspace{1mm}
\STATE $y \gets \text{Sample $i$ proportional to } \exp[\hat{\vy_i}\sigma/(2\gamma)]$
\vspace{1mm}
\STATE $\vz \leftarrow \vz + [y]$ \COMMENT{Append token to response}
\vspace{1mm}
\ENDFOR
\vspace{1mm}
\STATE \textbf{return} $\vz$
\end{algorithmic}
\end{algorithm}

\subsection{Theoretical Analysis}\label{sec:theoretical_analysis}

In addition to the experimental comparison, we provide a theoretical perspective on our Product of Experts view. We focus on the conditional independence assumption in our and previous work and show that it is less restrictive than it may appear once we examine how contexts are actually used. In the examples (see~\cref{sec:scenarios}), the agent does not see the entire private in-context set of examples at once (e.g., the whole inbox, the full code base, or the full image stream). Instead, it is exposed to a disjoint, often singular subset of context when answering a query: a few emails about the current topic, a few recent code diffs, or a few image-label pairs for a VLM evaluation. It is natural to model this selection step as drawing i.i.d.\ ``views'' from a much larger underlying private state: for instance, sampling a few emails from the distribution of inbox messages, or sampling a few crops or patches from a full image.

Formally, we model a high-dimensional latent state $\mS$ (the full inbox, repository, or dataset), and each in-context example $\mC_j$ is obtained by applying a randomized projection operator $P_j$ to $\mS$ (e.g., a random retrieval rule, a random crop, or a random subsequence). Different choices of $P_j$ give different examples, and it is reasonable to treat these views as independent conditional on $\mS$. Under this ``many small views of a large object'' model, each example is an independent low-dimensional summary of the same underlying state. The theorem below shows that, aggregating the influence of a number of such independent views, the resulting probability distribution converges to the distribution that we would have obtained if we could condition directly on the high-dimensional state $\mS$.

We now state the assumptions for the theorem.
The key is that, while transformer attention creates cross-example interactions, we assume these interactions remain bounded as the number of examples grows, whereas the contribution from individual examples accumulates.

\begin{assumption}
\textbf{Bounded interaction}
\label{ass:bounded-interaction}
For a given query $\vx$ and output token $y$, we decompose the unnormalized LLM log-probability distribution as follows:
\begin{align}
\log \tilde{p}_{\mathrm{LLM}}( y \mid  \vx, \tpurple{\mC_{1:J}}) = \sum_{j=1}^J \psi(y, \tpurple{\mC_j}; \vx) + R_J(y, \tpurple{\mC_{1:J}}), \nonumber
\end{align}
where $\psi(y, \tpurple{\mC_j}; \vx)$ is the log-likelihood contribution of example $\tpurple{\mC_j}$, satisfying $|\psi(y, \mC; \vx)| \leq M$ for some $M < \infty$; $R_J(y, \mC_{1:J})$ is a residual capturing cross-example interactions, satisfying $|R_J(y, \mC_{1:J})| \leq B$ uniformly in $J$, $y$, and the choice of examples, for some $B < \infty$.
\end{assumption}

\begin{remark}[Interpretation of Assumption]
The assumption states that cross-example attention effects, in which the representation of $\mC_i$ is modified by attending to $\mC_j$, contribute a bounded amount to the final logits. This is plausible when: 1) Each example primarily attends to the query $\vx$ and to itself, with limited cross-example attention; 2) The model's in-context learning mechanism operates via a ``soft retrieval'' where each example independently predicts an output; 3) Attention weights to other examples are $O(1/J)$, so their total contribution remains bounded.

\end{remark}

\begin{theorem}[Convergence to Full-Context Distribution under Bounded Interaction]
\label{thm:bounded-interaction}
Let $\mS$ be a private state, and let $\mC_1, \mC_2, \ldots \overset{\mathrm{i.i.d.}}{\sim} P(\mC \mid \mS)$ be independent views. Suppose Assumption~\ref{ass:bounded-interaction} holds with constants $M$ and $B$. Define the normalized LLM probability distribution:
\begin{equation}
p_J(y \mid \vx, \mC_{1:J}) := \frac{\tilde{p}_{\mathrm{LLM}}(y \mid \vx, \mC_{1:J})}{\sum_{y'} \tilde{p}_{\mathrm{LLM}}(y' \mid \vx, \mC_{1:J})}.
\end{equation}
Define the limiting distribution:
\begin{equation}
p^\star(y \mid \vx, \mS) := \frac{p_{\mathrm{LLM}}(y \mid \vx) \exp\bigl(\bar{\psi}(y, \mS; \vx)\bigr)}{\sum_{y'} p_{\mathrm{LLM}}(y' \mid \vx) \exp\bigl(\bar{\psi}(y', \mS; \vx)\bigr)},
\end{equation}
where $\bar{\psi}(y, \mS; \vx) := \mathbb{E}_{\mC \sim P(\cdot \mid \mS)}[\psi(y, \mC; \vx)]$ is the expected marginal contribution. Then for almost every $\mS$:
\begin{equation}
\bigl\| p_J(\cdot \mid \vx, \mC_{1:J}) - p^\star(\cdot \mid \vx, \mS) \bigr\|_1 \xrightarrow{\mathrm{a.s.}} 0 \quad \text{as } J \to \infty.
\end{equation}
\end{theorem}

\begin{proof}
Define the score vector $u_J(y) := \frac{1}{J}\log \tilde{p}_{\mathrm{LLM}}(y \mid \vx, \mC_{1:J})$. By Assumption~\ref{ass:bounded-interaction}:
\begin{align}
u_J(y) &=  \frac{1}{J}\sum_{j=1}^J \psi(y, \mC_j; \vx) + \frac{1}{J}R_J(y, \mC_{1:J}).
\end{align}

We analyze each term.
\textbf{Sum term:} By the strong law of large numbers (SLLN), $\mC_j \overset{\mathrm{i.i.d.}}{\sim} P(\cdot \mid \mS)$ and $|\psi(y, \mC_j; \vx)| \leq M$:
    $\frac{1}{J}\sum_{j=1}^J \psi(y, \mC_j; \vx) \xrightarrow{\mathrm{a.s.}} \bar{\psi}(y, \mS; \vx).$
\textbf{Residual term:} Since $|R_J| \leq B$, $\left|\frac{1}{J}R_J\right| \leq \frac{B}{J} \to 0$.

Thus $u_J(y) \xrightarrow{\mathrm{a.s.}} \bar{\psi}(y, \mS; \vx)$ for each $y$. Since $\mathcal{Y}$ is finite, this convergence is uniform over $y$.
Now, the normalized distribution $p_J$ can be written as, with $y_{\max} = \arg\max_{y'} u_J(y')$:
\begin{align}
p_J(y \mid \vx, \mC_{1:J}) &= \frac{\exp(J \cdot u_J(y))}{\sum_{y'} \exp(J \cdot u_J(y'))} \nonumber \\
&= \frac{\exp(u_J(y) - u_J(y_{\max}))}{\sum_{y'} \exp(u_J(y') - u_J(y_{\max}))}.
\end{align}

Define $v_J(y) := u_J(y) - \max_{y'} u_J(y')$ and, similarly, $v^\star(y) := \bar{\psi}(y, \mS; \vx) - \max_{y'} \bar{\psi}(y', \mS; \vx)$. Since $u_J \to \bar{\psi}$ uniformly and taking the maximum is continuous on finite sets, we have $v_J \to v^\star$ uniformly.

The softmax map $v \mapsto \exp(v) / \sum_{y'} \exp(v(y'))$ is continuous. Therefore:
\begin{equation}
p_J(y) = \mathrm{softmax}(v_J)_y \;\xrightarrow{\mathrm{a.s.}}\; \mathrm{softmax}(v^\star)_y = p^\star(y \mid \vx, \mS) \nonumber.
\end{equation}
Continuity of the $\ell_1$ norm completes the proof.
\end{proof}

\begin{table*}[]
\centering
\setlength{\tabcolsep}{4pt}
\renewcommand{\arraystretch}{1.2}
\caption{Comparison against prior art. For each dataset, the accuracy of each method is significantly higher than the 0-shot result, demonstrating ICL. For both $8$ and $25$ in-context examples across three datasets, our method achieves significantly higher accuracy. The DP is set to $\varepsilon=4$, following prior work~\citep{tang2024dpfewshot}. The LLM is \qwen~and Table~\ref{tab:repr_bigtable} reproduces this table with \llama.}
\begin{tabularx}{\textwidth}{lrrrclll}
\toprule
 & $\varepsilon$ & \textbf{Num data} & \textbf{Method} & \textbf{Lbl. Priv.} & \textbf{AGNews} & \textbf{DBPedia} & \textbf{TREC} \\
\midrule
No context & 0 & 0 & Empty & \checkmark & $65.4_{ \pm 0.8}$ & $71.2_{ \pm 1.0}$ & $68.4_{ \pm 0.7}$ \\
Synthetic ~\citep{tang2024dpfewshot} & 4 & $10^5$ & Synth. data & \texttimes & $79.8_{\pm 0.7}$ & $81.9_{\pm 0.6}$ & $76.3_{\pm 1.1}$ \\
\midrule
\qquad \textit{8 context examples:} & & & & & & & \\
Privacy-by-Sampling ~\citep{wu2023privacyhistogramllm} & 4 & 8 & Subsampling & \checkmark & $80.7_{ \pm 0.5}$ & $70.3_{ \pm 0.9}$ & $72.2_{ \pm 0.9}$ \\
RNM ~\citep{wu2023privacyhistogramllm} & 4 & 8 & Hard voting & \checkmark & $65.5_{ \pm 0.6}$ & $46.3_{ \pm 0.8}$ & $52.4_{ \pm 0.9}$ \\ \rowcolor{LightCyan}
Product-of-Experts (ours) & 4 & 8 & Soft pred. & \checkmark & $\textbf{86.3}_{ \pm 0.5}$ & $\textbf{87.3}_{ \pm 0.6}$ & $\textbf{78.9}_{ \pm 0.5}$ \\
\qquad \textit{25 context examples:} & & & & & & & \\
Privacy-by-Sampling ~\citep{wu2023privacyhistogramllm} & 4 & 25 & Subsampling & \checkmark & $81.5_{ \pm 0.4}$ & $75.2_{ \pm 1.1}$ & $68.1_{ \pm 0.7}$ \\
RNM ~\citep{wu2023privacyhistogramllm} & 4 & 25 & Hard voting & \checkmark & $85.3_{ \pm 0.6}$ & $85.7_{ \pm 0.7}$ & $76.7_{ \pm 0.6}$ \\ \rowcolor{LightCyan}
Product-of-Experts (ours) & 4 & 25 & Soft pred. & \checkmark & $\textbf{87.0}_{ \pm 0.5}$ & $\textbf{88.0}_{ \pm 0.6}$ & $\textbf{78.8}_{ \pm 0.5}$ \\
\midrule
\textcolor{gray}{Non-private, no DP applied} & \textcolor{gray}{$\infty$} & \textcolor{gray}{8} & \textcolor{gray}{In-context} & \textcolor{gray}{-} & \textcolor{gray}{$87.5_{ \pm 0.5}$} & \textcolor{gray}{$89.9_{ \pm 0.5}$} & \textcolor{gray}{$80.5_{ \pm 0.8}$} \\
\textcolor{gray}{Non-private, no DP applied} & \textcolor{gray}{$\infty$} & \textcolor{gray}{25} & \textcolor{gray}{In-context} & \textcolor{gray}{-} & \textcolor{gray}{$\textbf{87.8}_{ \pm 0.6}$} & \textcolor{gray}{$\textbf{92.0}_{ \pm 0.7}$} & \textcolor{gray}{$\textbf{81.0}_{ \pm 0.7}$} \\
\bottomrule
\end{tabularx}
\label{tab:bigtable}
\end{table*}

\begin{table}[h]
\centering
\setlength{\tabcolsep}{3pt}
\renewcommand{\arraystretch}{1.3}
\caption{Results on arithmetic tasks in the GSM8k dataset. All evaluations have higher accuracy than the 0-shot accuracy at 14\% -- showing the ICL benefit. Across varying numbers of ICL examples, our method achieves significantly higher accuracy than RNM.}
\begin{tabular}{lr r r}
\toprule
& \multicolumn{3}{l}{\textbf{Number of examples}} \\
 & \textbf{4} & \textbf{8} & \textbf{20} \\
\midrule
\textcolor{gray}{Non-private, no DP} & \textcolor{gray}{$44.2_{ \pm 1.0}$} & \textcolor{gray}{$44.5_{ \pm 0.8}$} & \textcolor{gray}{$46.0_{ \pm 0.9}$} \\
\midrule
RNM~\citep{wu2023privacyhistogramllm} & $15.7_{ \pm 0.6}$ & $20.3_{ \pm 0.5}$ & $36.5_{ \pm 0.6}$ \\ \rowcolor{LightCyan}
PoE (Ours) & $\textbf{37.3}_{ \pm 1.3}$ & $\textbf{40.9}_{ \pm 1.1}$ & $\textbf{43.1}_{ \pm 1.1}$ \\
\bottomrule
\end{tabular}
\label{tab:gsm8k_acc}
\end{table}

\begin{remark}[Rate of convergence]
The proof reveals that convergence occurs at the rate $O(1/\sqrt{J})$ from the SLLN plus $O(1/J)$ from the vanishing residual and prior terms. In practice, this suggests that moderate values of $J$ (e.g., $J = 20$--$50$) should suffice for near-convergence.
\end{remark}

\begin{remark}[Connection to our algorithm]
Multiple privacy-preserving Algorithms, including ours, query each example independently, i.e., computing $p(y \mid \mC_j, \vx)$ separately. Under Assumption~\ref{ass:bounded-interaction}, this independent querying discards only the bounded residual $R_J$, which vanishes per Theorem~\ref{thm:bounded-interaction} as $J$ grows.
The theorem shows that, conditional on the boundedness assumption, the error introduced by independent querying vanishes as $J$ grows. 
This provides a theoretical framework for understanding when independent querying might be appropriate, and we present experimental results in Section~\ref{sec:experiments} that support the assumption, even without the noise required for DP. 
\end{remark}

\subsection{Methodological comparison to prior approaches}

In the experimental results, we compare with three prior approaches. The first approach is named after the Report-Noisy-Max (RNM) method~\citep{dwork2014algorithmic}. The method relies on the same conditional independence assumption~\citep{wu2023privacyhistogramllm}, but uses thresholded ``hard predictions,'' which discard the nuance of the predictive likelihoods. This would be achieved by replacing line 7 in Algorithm~\ref{alg:dp_icl_ci} with $\hat{\vy} \gets \hat{\vy} + \text{one-hot}(\argmax \vl)$, which means that $\hat{\vy}$ is updated with a vector of all zeros and a single one where the largest likelihood occurs. The RNM then relies on observing that such a vector has sensitivity 1, and adds proportionally scaled noise; we refer to their paper for the proof of the constants.

Another approach towards the same DP guarantee is `Privacy-Amplification-by-Subsampling' (PbS), which also preserves the `uncertainty nuance' of predictions. Similar to DP-SGD~\citep{abadi2016deep}, this method repeatedly samples subsets of the context to make noisy predictions, which are then averaged. This was suggested by~\citet{wu2023privacyhistogramllm}, and we compare to it in the experimental section. A major downside of this method is that it requires significant computational resources. Due to the Poisson sampling assumption~\citep{chua2024private}, potentially the entire context could require computation in each random sample, of which there could be up to hundreds~\citep{wu2023privacyhistogramllm}.

\section{Experimental results}\label{sec:experiments}

The experimental evaluation and comparison with related work are conducted across five datasets in three categories: text classification, grade-school math, and vision-language. This section also includes results from a Membership Inference Attack to explore the attack scenario empirically. 

\subsection{Comparison of Hard or Soft predictions}\label{sec:hard_or_soft}

First, we compare the thresholded against the soft predictions that motivate our work. Soft predictions preserve nuance in every prediction. For example, in a two-class prediction problem, two predictions might be \num{0.6} and \num{0.7} for one class, \num{0.4} and \num{0.3} for the other. PoE will work with log-probabilities $[\log 0.6 + \log 0.7, \log 0.4 + \log 0.3]$, instead of thresholded $[2, 0]$ used for RNM.

To empirically justify this approximation, we use the predictive likelihoods from the independent LLMs' predictions and compare the vectors from Hard Voting and Soft Predictions. Let's say one prediction over $K$ classes would be $\vp = [p_1, p_2, \cdots, p_K] \in [0,1]^K$. For hard voting, this would be thresholded to $\vh \in \{0,1\}; \sum \vh = 1$. For soft predictions, this would be clipped to $\vs  = [s_1, s_2, \cdots, s_K] \in [e^{-\gamma},1]^K$.

We compare the predictive likelihood at the target label $l_{\text{mean}}(\vq, y) = \vq_{y}$ for approximate vector $\vq$ at label $y$. $\vq$ could be either the hard thresholded $\vh$ or the clipped $\vs$. The predictive likelihood is important because it is the probability that this expert alone will sample the target label class. Secondly, we compare the $\ell_\infty$-norm $D_{\infty}(\vp, \vq) = \max_i |p_i - q_i|$ norm, which indicates the largest change in sampling probability between the unclipped vector, $\vp$, and thresholded or clipped $\vq$. To this end, we sample 3000 context and query examples from the GSM8k classification task~\citep{gsm8k} and make predictions with a Qwen3 model~\citep{qwen3technicalreport}. Compared between hard and soft predictions, the predictive likelihood increases from $42.3_{\pm 0.6}$\% to $45.3_{\pm 0.9}$\% and the $\ell_\infty$-norm decreases from $39.2_{\pm 0.5}$\% to $13.5_{\pm 0.1}$\%. This shows that the mean predictive likelihood and the $\ell_\infty$-norm are much better for soft predictions. These results motivate our method, since DP methods must sample from those private states.

\subsection{Text classification}

\begin{figure}[t]
    \centering
    \includegraphics[width=\linewidth]{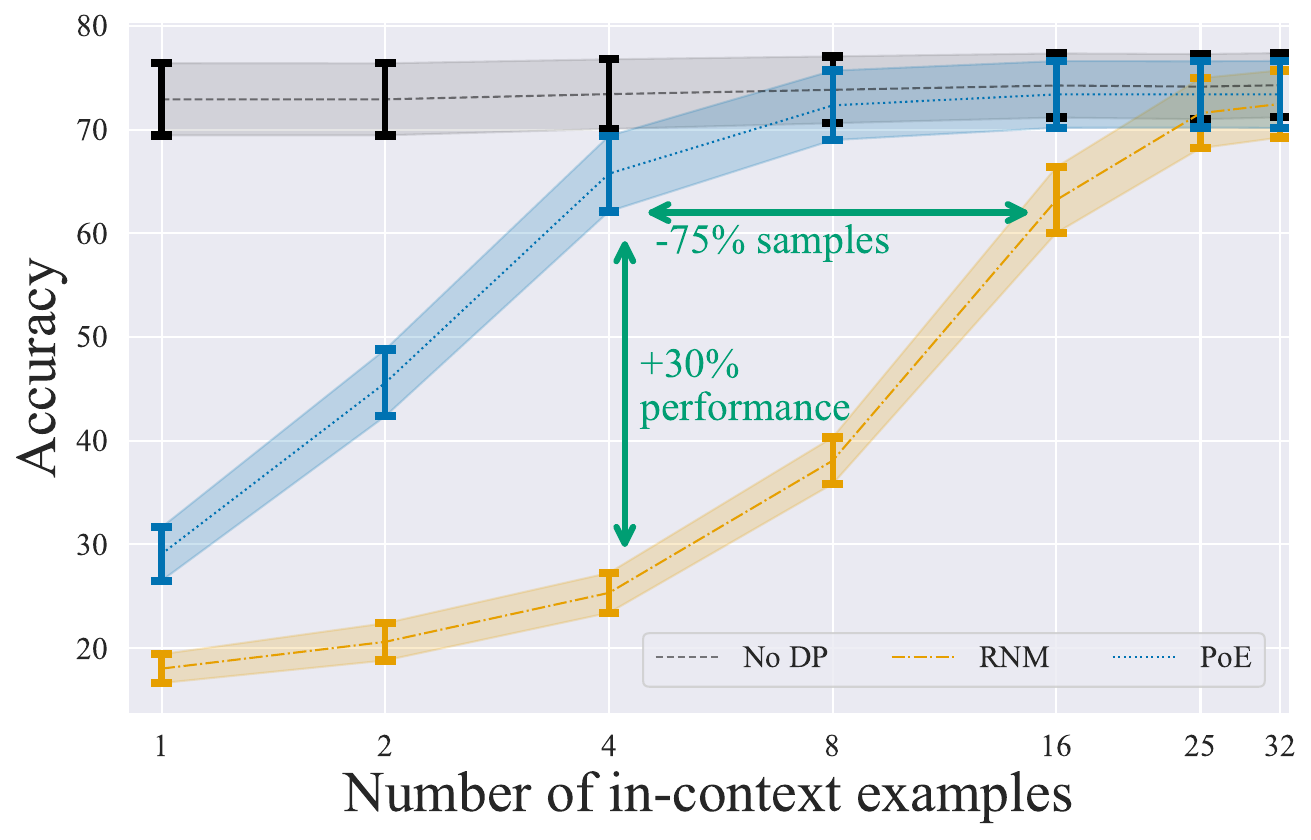}
    \caption{Average accuracy across AGNews, DBPedia, TREC, and GSM8k datasets. Our method performs significantly better than previous work -- especially for a small number of examples, where ICL is widely used. For $J=4$ examples, the improvement in accuracy is 30\% points on average. To achieve the same accuracy, RNM would need almost 4x as many in-context examples.}
    \label{fig:allfour_accuracy}
\end{figure}

\textbf{Text classification}
The evaluation tasks involve 4-way news classification on \textsc{AGNews}~\citep{zhang2015character}, 6-way question categorization on the \textsc{TREC} dataset~\citep{voorhees2000building}, and 14-way entity classification on the \textsc{DBPedia} dataset~\citep{zhang2015character}. This evaluation follows the setup of~\citet{zhao2021calibrate,tang2024dpfewshot}, and we use the prompt settings from the published code of~\citet{tang2024dpfewshot}. Their models are not available, so we evaluate their code with a \qwen~model. Appendix~\ref{app:extra_results} reproduces this experiment with a \llama~model.
The success metric is classification accuracy. Each accuracy is a mean (and $\pm$ standard error) of 25 random seeds; each random seed is the evaluation of 100 random prompts from the test set against a randomly drawn context set. To be comparable with previous work, evaluations for this task are done at $\varepsilon=4$. Where Approximate DP applies, $\delta=10^{-5}$ is used.  More details are provided in Appendix~\ref{app:exp-details-general}.

Table~\ref{tab:bigtable} shows the results of the text classification experiments. It is important to note that the synthetic data method of ~\citet{tang2024dpfewshot} provides a privacy guarantee against a training dataset of about \num{100000}+ examples. In contrast, for both PbS, RNM, and our method, the privacy guarantee holds against the actual context set, as explained in Appendix~\ref{app:threat_model_formal}. Despite this difference, the synthetic data method achieves lower accuracy than our method.

Across all three text classification datasets, our method achieves higher accuracy than other methods. Especially for a small number of examples, which is an important setting for in-context learning~\citep{gsm8k,ICL_VLM_pseudo-name}, our method scores more than 20\% points higher accuracy than RNM and more than 6\% points higher than PbS. This is likely due to the uncertainty information preserved in the soft predictions, c.f. Section~\ref{sec:hard_or_soft}.

\begin{table}[t]
\centering
\setlength{\tabcolsep}{3pt}
\renewcommand{\arraystretch}{1.3}
\caption{Results on the VLM pseudoname labelling task. Across varying numbers of in-context examples, our method achieves higher accuracy. For comparison, also the non-private but conditionally independent results are included; 0-shot accuracy is \num{20}\%.}
\begin{tabular}{l c c c}
\toprule
& \multicolumn{3}{c}{\textbf{Number of examples}} \\
 & \textbf{4} & \textbf{8} & \textbf{20} \\
\midrule
Non-private, no DP & $89.0_{\pm 0.6}$ & $89.1_{\pm 0.6}$ & $89.8_{\pm 0.6}$ \\
Soft cond. indep. & $82.6_{\pm 0.8}$ & $82.7_{\pm 0.8}$ & $83.3_{\pm 0.8}$ \\
Hard cond. indep. & $76.0_{ \pm 0.9}$ & $80.7_{ \pm 0.8}$ & $82.6_{ \pm 0.8}$ \\
\midrule
RNM~\citep{wu2023privacyhistogramllm} & $28.8_{ \pm 0.9}$ & $39.4_{ \pm 1.0}$ & $69.8_{ \pm 0.9}$ \\ \rowcolor{LightCyan}
PoE (Ours) & $\textbf{35.1}_{ \pm 1.0}$ & $\textbf{53.7}_{ \pm 1.0}$ & $\textbf{78.7}_{ \pm 0.8}$ \\
\bottomrule
\end{tabular}
\label{tab:vlm_accuracy}
\end{table}

\textbf{Math evaluation: } The math evaluation is done on the \textsc{GSM8k} dataset~\citep{gsm8k}. This dataset consists of math or arithmetic questions aimed at the level of grade-schoolers. We turn this into a 10-way classification task by calculating accuracy on getting the first digit of the answer correct. Although this evaluation is not standard, the accuracy is significantly higher than the default 10\% accuracy, demonstrating the LLMs' mathematical capabilities. 

Table~\ref{tab:gsm8k_acc} displays the results of the math evaluation on GSM8k. Across a number of examples, our method achieves higher accuracy than RNM. This aligns with the text classification results. Also, for a smaller number of examples, the improvement over RNM is more pronounced, which is an important aspect of in-context learning.

\textit{Results across a varying number of examples }
To assess performance across the full range of in-context learning use cases, we evaluate all four text classification datasets across a varying number of examples (Figure~\ref{fig:allfour_accuracy}). The practically relevant regime for ICL is 4--8 examples~\citep{min2022rethinking}, where few-shot learning significantly improves over zero-shot performance. Here, our method outperforms RNM by 30\% points at 4 examples, and to achieve comparable accuracy, RNM requires 4x as many context examples (16). This shows that our method is effective at improving accuracy in the practically relevant regime for in-context learning.

\subsection{Vision-language}

To demonstrate that our method is not only effective with text tokens but also applicable across modalities, we conduct a final experiment using a vision-language model. For simplicity, we focus on the task of pseudo-name labelling~\citep{ICL_VLM_pseudo-name}. In this task, a Vision-Language Model (VLM) is shown images and randomly assigned non-English names to them. The VLM will then be prompted with an image, and it must classify the image's name. It is a 5-way classification task. We choose this pseudo-name labelling setting so that the VLM cannot rely on pretraining knowledge and must learn from in-context examples. Previous work has established that this task is suitable for in-context learning~\citep{ICL_VLM_pseudo-name,derakhshani2023self_secat}. Similar to earlier experiments, the error bars in Table~\ref{tab:vlm_accuracy} are the standard error of the mean accuracy over 2500 random seeds, explained in Appendix~\ref{app:exp-details-general}.

Table~\ref{tab:vlm_accuracy} displays the results of the vision-language evaluation. These results were obtained with $\varepsilon=1$-DP, which is the recommended setting by~\citet{choosing_epsilon} and ~\citet{choosing_epsilon_02}. Across a number of examples, our method achieves higher accuracy than RNM. This aligns with the text classification and math evaluation results. Also, for a smaller number of examples, which are widely used in ICL, the improvement over RNM is more pronounced.

\textit{Comparison of the conditional independence assumption: } We also investigate the difference between Hard and Soft predictions in a classification setting. To this end, Table~\ref{tab:vlm_accuracy} shows the accuracies when only the conditional independence assumption is applied, but the noise required for DP is not yet added. This means that for Hard predictions, only the discrete votes are summed, and the class with the most votes is predicted. For soft predictions, this means that the clipped predictions are multiplied, and the class with the highest likelihood is predicted. For both the Hard and the Soft versions, the accuracy is lower than the non-private result, which is expected, as we introduce a strict assumption. However, the decrease in accuracy is much larger for Hard predictions than for soft predictions. Moreover, this decrease is much worse for a small number of examples -- likely because the uncertainty in predictive likelihood is more important. This shows that, even in the absence of DP noise addition or noisy sampling, Hard predictions have a worse impact on the final accuracy than soft predictions.

\subsection{Membership Inference Attack} \label{sec:mia}

\begin{table}[t]
    \centering
    \setlength{\tabcolsep}{3pt}
\renewcommand{\arraystretch}{1.3}
\caption{Empirical privacy vulnerability as measured by a Membership Inference Attack. These results show a significant attack vulnerability (AUROC 56-93\%), and DP significantly reduces the vulnerability, i.e., an AUROC close to the target value of 50\%.}
\begin{tabularx}{\linewidth}{X X X X X}
\toprule
 & AGNews & DBpedia & TREC & GSM8k \\
\midrule
No DP & 60.0 & 56.9 & 63.6 & 93.9 \\
$\varepsilon=1$ DP & 52.9 & 49.8 & 53.8 & 53.5 \\
\bottomrule
\end{tabularx}
\label{tab:mia}
\end{table}

Complementary to the analytical privacy guarantees, we run a Membership Inference Attack (MIA)~\citep{shokri2017membership} to obtain empirical results on privacy. Although a MIA cannot prove that the model uses a particular data point~\citep{zhang2024membershipinferenceattacksprove}, we investigate if the original ICL is vulnerable to an empirical privacy attack and to what extent our DP guarantee reduces this vulnerability. The MIA aims to output a score, ${s}$, indicating whether a datapoint is a member of the context examples. Several MIAs have been reported in the literature~\citep{camia,carlini_ref_model_MIA}, and we use the attack based on Likelihood Ratios (LiRA). 
We use likelihood-based MIA as a proxy for privacy vulnerability on sampled labels. In our case, the summed log-likelihood is the score ${s}$. While label-only MIAs exist~\citep{he2025labelonlyMIA}, they have only been shown to work for longer sequences by approximating soft likelihood. For simplicity, we use the proxy directly.
A comparatively high score indicates membership, and this is measured using the Area Under the ROC curve (AUROC) metric. The AUROC is the area under the curve on a plot where the x-axis is the False Positive Rate (flagging a member when it is not) and the True Positive Rate (flagging a member when it is). Appendix~\ref{app:exp-details-mia} provides more details. 

The results of the MIA are in Table~\ref{tab:mia}. The unprotected ICL setting is vulnerable to a membership inference attack, as indicated by an AUROC of \num{60.0} on the AGNews dataset, for example. 
Subsequently, the DP ICL has an AUROC of \num{52.9} on AGNews, indicating better privacy protection. The same trend, where the DP model is much closer to \num{50.0}, is observed for the other datasets. This shows that, in addition to the theoretical privacy guarantees, DP ICL is also empirically more private against MIAs.

\section{Conclusion}\label{sec:conclusion}

We introduce a novel approach to differentially private in-context learning that reformulates ICL through the lens of a Product-of-Experts (PoE) approximation. Our key contribution is the shift from \textit{hard voting} to \textit{soft prediction} aggregation. This enables our method to leverage the full uncertainty information in the model's predictive probabilities. This uncertainty information is shown to be beneficial compared to prior approaches that discard the uncertainty through discrete thresholding. The soft prediction approach, where log-probabilities across in-context examples are clipped and aggregated, proves particularly effective for a small number of examples (four to eight examples), which is the common setting in many real-world ICL applications.

Our experimental evaluation demonstrates consistent improvements across five datasets including text classification, math, and vision-language tasks. On average, we observe a 30\% point improvement for text classification with only 4 examples.
Beyond empirical performance, we provide both theoretical and empirical analyses of privacy. Theoretically, we propose an algorithm that guarantees Differential privacy. Additionally, we establish that under a bounded-interaction assumption, our Product-of-Experts approximation converges to the full-context distribution as the number of examples increases. Empirically, we complement the privacy guarantee with a Membership Inference Attack on four datasets. 
Overall, the improved privacy-utility trade-off makes our method practical for deployment in agentic AI and RAG settings where privacy-sensitive local data must be leveraged without compromising an individual's privacy.

\section*{Acknowledgements}

This work is financially supported by Qualcomm Technologies Inc., the University of Amsterdam and the allowance Top consortia for Knowledge and Innovation (TKIs) from the Netherlands Ministry of Economic Affairs and Climate Policy. CL is with Qualcomm AI research, which is an initiative of Qualcomm Technologies, Inc. and/or its subsidiaries. Correspondence may go to r.romijnders@uva.nl.

\section*{Impact Statement}

Our work advances private machine learning by enabling more effective privacy-preserving in-context learning. DP could have a positive societal impact by allowing systems to use local  private information while providing formal privacy guarantees. The improved privacy-utility trade-off makes DP ICL more accessible for real-world deployment.

However, while DP provides rigorous guarantees, organizations must remain thoughtful about its deployment. We warrant that DP cannot be used as a blanket justification for additional data access -- although this effect remains an open question~\citep{brough2022bulletproof}. Each deployment of any ICL variant should justify processing the data in contexts where data minimization would also be an option.

Second, we acknowledge that differentially private methods can have disparate impacts across different groups in the data. Prior research has demonstrated that privacy-preserving mechanisms may disproportionately affect model performance on minority classes and underrepresented populations~\citep{farrand2020neither,DPDisparateImpact}. While our theoretical and empirical contributions improve overall privacy-utility trade-offs, they do not directly address these fairness concerns. Future work should investigate how the Product-of-Experts approximation interacts with fairness objectives and develop techniques to ensure equitable utility across demographic groups.

\bibliography{main_bib}
\bibliographystyle{icml2026}

\newpage
\appendix
\onecolumn

\section{Notation}\label{app:notation}

We briefly outline the notation used in the main paper and in the following appendices.

\begin{itemize}[itemsep=0mm]  
    \item $\gamma$ is the clipping bound for the log-probabilities;
    \item $\sigma$ is the noise multiplier for the exponential mechanism. This is tightly related to the $\varepsilon$ parameter for Differential Privacy. For example, in Algorithm~\ref{alg:dp_icl_ci}, $\sigma = \frac{1}{T}\varepsilon$ when the Product-of-Experts is used with $T$ steps and naive composition;
    \item $T$ is the maximum number of tokens in the response;
    \item $J$ is the number of in-context examples;
    \item $\mC_j$ is the $j$-th in-context example;
    \item $\vx$ is the query, it is a sequence of tokens;
    \item $\vy_{<t}$ is the response up to and including time $t-1$;
    \item $\vy_t$ is the token at time $t$, sometimes the subscript is dropped to have just $y$; $\vy_{1:T}$ indicates the response sequence for all $T$ tokens;
    \item $\vl$ is a vector of log-likelihoods, for example, when in~Algorithm~\ref{alg:dp_icl_ci}: ``$\vl \gets LM(\cdot | \vz, \vx, \tpurple{\mC_j})$,'' then $\vl$ is a vector of length equal to the vocabulary of the LLM where each value $\vl_i$ is the log-likelihood for token $i$;
    \item $u(y, \mC_j)$ is the utility function for the exponential mechanism. It is the log-probability of the response $y$ given a \textbf{single} in-context example $\mC_j$;
    \item $U(y, \mC_{1:J})$ is the utility function when \textbf{all} the in-context examples $\mC_{1:J}$ are considered; this distinction will be made more clear in Appendix~\ref{app:grouping};
    \item $\text{clip}_\gamma(\vl)$ is the clipping function for the log-probabilities. It sets all values outside the interval $[-\gamma, 0]$ to zero;
    \item $\text{vclip}_\gamma(\vl)$ in the section on Gaussian Mechanism is the vector clipping function $\text{vclip}_\gamma(\vx) = \vx / \max(\gamma, \|\vx\|_2)$;
    \item $Z$ is, unless otherwise specified, a constant of proportionality -- a normalization constant;
    \item $\varepsilon$ and $\delta$ are constants that define the Differential Privacy guarantee; generally, $\varepsilon \leq 1$ and $\delta \leq 10^{-5}$ are preferred. Only for Tables~\ref{tab:bigtable} and~\ref{tab:repr_bigtable} do we use $\varepsilon = 4$, to follow previous work and their implementation details~\citep{tang2024dpfewshot};
    \item $\mD$ is a dataset, so a collection of documents, denoted by $\mC_j$ for $j=1, \ldots, J$.
\end{itemize}

Apart from notation, we clarify that with `prompt' we mean the entire textual input to an LLM, which consists of both a query and context examples. Abstractly written, $\text{prompt} = \text{query} + \text{context examples}$.

\section{Details of the Differential Privacy Analysis}\label{app:priv-details}

\subsection{Privacy for Algorithm 1}\label{app:priv-proof}

This section is an extended proof that Algorithm~\ref{alg:dp_icl_ci} satisfies Differential Privacy for defined settings of $\sigma$ and $T$. We start from the Product-of-Experts structure, repeated from Equation~\ref{eqn:poe_break}, $p(y_t | \vy_{<t}, \vx , \tpurple{\mC_{1:J}}) \propto  \prod_j p(\vy_t |  \vy_{<t}, \vx , \tpurple{\mC_j})$.

This structure is chosen such that the contribution of each context example, $\mC_j$, is isolated. As such, we repeat the definition of Differential Privacy from~\citet{dwork2014algorithmic}.

\hspace{5mm}\fbox{\begin{minipage}{30em}
A randomized algorithm $A(\cdot)$ is $\varepsilon$-differentially private if the following holds for any two adjacent datasets $\mD$, $\mD'$, and for any subset $\mathcal{S}$ of outputs:
\begin{align}
    \Pr[A(\mD) \in \mathcal{S}] \leq e^\varepsilon \Pr[A(\mD') \in \mathcal{S}]  \label{eqn:definition_dp},
\end{align}
where $\mD = \{ \mC_j \}_{j=1}^J$ is a dataset that contains $J$ documents and $\mD'$ is the same dataset where at most one sample is different. This adjacency is further explained in Section~\ref{app:threat_model_formal}. In some cases, to Equation~\ref{eqn:definition_dp} a small constant $\delta$ is added which is usually $10^{-6}$. The definition is then named ``$(\varepsilon, \delta)$ (Approximate) Differential Privacy.''
\vspace{1mm}
\end{minipage}}

In our case, this translates to \textit{``For any set of in-context examples, if only one example were different, the probability of the predicted class would differ by at most a factor $e^\varepsilon$.''}

Our randomized algorithm is the Exponential Mechanism from~\citet{dwork2014algorithmic}, and works as follows. There is a token $y$, from a discrete set of tokens $\mathcal{Y}$. Depending on the context, $\mC_{1:J}$, there is a utility function $u(y, \mC_j)$. This utility follows from Equation~\ref{eqn:prod_experts_sum_vectors}, which we repeat here.

\begin{align}
    \log \prod_{j=1}^{J} p(\vy_t |  \vy_{<t}, \vx, \tpurple{\mC_j} ) = \sum_{j=1}^{J} \log p(\vy_t | \vy_{<t}, \vx, \tpurple{\mC_j} ).
\end{align}

As such, the utility reflects $\log p(\vy_t |  \vy_{<t}, \vx, \tpurple{\mC_j} )$. The sensitivity, then, can be bounded by the clipping operation. We take the likelihood factor $l = u(\vy_t, \mC_j) = \log p(\vy_t | \vy_{<t}, \vx, \tpurple{\mC_j} )$ and the clipping function $\text{clip}_\gamma(\vl)$:

\begin{align}\label{eqn:clipping_def}
  \text{For scalar $l$:   }  \text{clip}_\gamma(l) = \begin{cases}
        l & \text{if} \ \ -\gamma \leq l \leq 0 \\
        0 & \text{otherwise}
    \end{cases},\qquad \quad
\text{For vector $\vl$:   }\text{clip}_\gamma(\vl)_i = \begin{cases}
        \vl_i & \text{if} \ \ -\gamma \leq \vl_i \leq 0 \\
        0 & \text{otherwise}
    \end{cases}
\end{align}

The upper clipping at 0 is trivially achieved, as the log-likelihood is always negative because, for a discrete domain, the probability assigned to a class is less than or equal to 1. This maps to the definition of sensitivity in~\citet{dwork2014algorithmic} and we use their Exponential Mechanism with the following distribution:
\begin{align}
    \Pr[y] \propto e^{\frac{\varepsilon}{2\gamma} \sum_{j=1}^J u(y, \mC_{j})}
\end{align}

In this notation, $\sigma = \varepsilon$ when the Product-of-Experts is used with $T=1$.

\textbf{Accounting } The $\sigma$ in Algorithm~\ref{alg:dp_icl_ci} could be used for privacy accounting over the multiple tokens $T$. By composition rules of DP $\varepsilon = T \sigma$, which evenly divides the $\varepsilon$ among the $T$ time steps~\citep{dwork2014algorithmic}. One can use the Advanced Composition Theorem of Differential privacy~\citep{kairouz2015composition} if one is willing to switch to ``Approximate Differential Privacy:'' $\varepsilon = T\sigma(e^{\sigma}-1) + \sigma \sqrt{2T \log \delta^{-1}}$ for a desired setting of $\delta$. This composition theorem can be advantageous because of the $\sqrt{T}$ scaling rather than linear scaling.
In case of the Report-Noisy-Max, we refer to the original publication~\citep{wu2023privacyhistogramllm} for the relation between $\sigma$, $T$, and $\varepsilon$ and $\delta$. The experiments in Section~\ref{sec:experiments} do not use accounting since the classification predictions have single-token outputs.

\subsection{Grouping in L2 sensitivity}\label{app:grouping}

In related work~\citep{tang2024dpfewshot}, as well as ours, a method is used that we call ``grouping.'' We define and explain it in this section for both our method and highlight how it was used in previous work.

When clipping utilities for the Exponential Mechanism, or when clipping vectors as in the Gaussian Mechanism, it is common practice to align the clipping unit with the privacy unit~\citep{abadi2016deep,chua2024mind}. For document-level privacy, for example, this would mean that the contribution of each document would be clipped separately. However, one can ``group'' privacy units into clipping units and maintain the proof of Differential Privacy. We outline this method for the exponential mechanism and the Gaussian Mechanism.

\textbf{Exponential Mechanism }

The exponential mechanism is defined as: $\Pr[y] \propto e^{\frac{\varepsilon}{2\Delta} U(y, \mD)}$, where $y$ is the token, and $U(y, \mD)$ is the utility function for the full dataset $\mD$. The sensitivity is:
\begin{align}
    \Delta = \max_{\mD \sim \mD'} |U(y, \mD) - U(y, \mD')|  \ \ \ \forall y \in \mathcal{Y}
\end{align}

Then, sampling a token $y$ is $\varepsilon$-differentially private.
Now, assume there is an arbitrary loss function $u(y, \mD_i)$, then one can ensure bounded utility by clipping the loss function, where the clipping function is defined in Equation~\ref{eqn:clipping_def}:
\begin{align}
    u(y, \mD) \coloneqq \sum_{i \in |\mD|} \text{clip}_\gamma(u'(y, \mD_i))
\end{align}

Assume, for now, that the utility of the dataset is the sum of utilities per sample:
\begin{align}
    U(y, \mD) = \sum_{i \in |\mD|} u(y, \mD_i)
\end{align}

Then, the sensitivity is $\Delta = \max_{\mD \sim \mD'} |U(y, \mD) - U(y, \mD')| \leq \gamma$.

A grouping, however, does not affect the sensitivity. Let's examine, without loss of generality, a group size of 3 and assume that the dataset size is divisible by 3. Then, we can define the utility function:

\begin{align}
    U(y, \mD) \coloneqq \sum_{k \in |\mD| / 3} \text{clip}_\gamma(U'(y, \mD_{3k-2}, \mD_{3k-1}, \mD_{3k}))
\end{align}

Even for the grouped utility function, the sensitivity still holds:

\begin{align}
    \Delta = \max_{\mD \sim \mD'} |U(y, \mD) - U(y, \mD')| \leq \gamma.
\end{align}

This grouping thus does not affect sensitivity. It does, however, introduce a new balance. Generally, the amount of noise added is proportional to the clipping constant, $\gamma$, and the signal is proportional to the number of in-context examples, $J$. When there are $J$ in-context examples, the signal-to-noise ratio is $\frac{J\gamma}{\gamma} = J$. When there are only $K = \frac{J}{3}$ utilities to be summed, the signal-to-noise ratio would be $\frac{K\gamma}{\gamma} = \frac{1}{3}J$, which is only one-third of its original value and can negatively impact predictive accuracy. The other side of the balance, though, is that predictions can be better when an LLM or VLM draws on multiple examples rather than a single one. In that case, the better prediction might outbalance the extra noise from the worse signal-to-noise ratio. We use group size 2 for all VLM experiments (and not for the text classification experiments) and report a small hyperparameter experiment for its effects in Section~\ref{app:exp-details-general}.

\textbf{Gaussian Mechanism } Grouping is also applicable for the Gaussian Mechanism, which is used in the source code of~\citet{tang2024dpfewshot}.
In this case, the dataset is a set of vectors $\mD = \{\vx_1, \vx_2, \cdots, \vx_N \}$.
An adjacent dataset has one vector changed, $\vx'_i$, and so $\mD' = \{\vx_1, \vx_2, \cdots, \vx'_i, \cdots, \vx_N \}$. Then, $\mD \sim \mD'$ denotes adjacent datasets.

For the Gaussian Mechanism, each vector is clipped so that the vector sum has a bounded sensitivity.
For vectors, we define the vector clipping function as $\text{vclip}_\gamma(\vx) = \vx / \max(\gamma, \|\vx\|_2)$.

Denote with $\vm_n = \text{vclip}_\gamma(\vx_n)$ the clipped vector. Then the sum $\vy = \sum_{n=1}^N \vm_n = \sum_n \text{vclip}_\gamma(\vx_n)$ has a bounded sensitivity.

\begin{align}
    \Delta_{\lTwo} = \max_{\mD \sim \mD'} \|\vy(\mD) - \vy(\mD')\|_2 \leq 2\gamma
\end{align}

The factor 2 arises here because a vector can point in all directions, and thus the sensitivity is twice the radius. In the Exponential Mechanism, the utility is monotonic, as defined in~\citet{dwork2014algorithmic}, so the factor 2 does not apply.

Without loss of generality, assume that the groupsize is three and that the $N$ is divisible by three.
We then have $K = N/3$ tuples: $\vz_k = \text{vclip}_\gamma(\vx_{3k-2} + \vx_{3k-1} + \vx_{3k})$.
Then, we care about $U(\mD) = \sum_{k=1}^K \vz_k = \sum_k \text{vclip}_\gamma(\vx_{3k-2} + \vx_{3k-1} + \vx_{3k})$ and its sensitivity:

\begin{align}
    \Delta_{\lTwo} = \max_{\mD \sim \mD'} \|U(\mD) - U(\mD')\|_2
\end{align}

Without loss of generality, assume that index $\vx_1$ is changed to $\vx'_1$. Then, the sensitivity is:
\begin{align}
\Delta_{\lTwo} &= \max_{\mD \sim \mD'} \|U(\mD) - U(\mD')\|_2 =
\max_{\vx_1 \sim \vx'_1} \| \text{vclip}_\gamma(\vx_1 + \vx_2 + \vx_3) - \text{vclip}_\gamma(\vx'_1 + \vx_2 + \vx_3) \|_2 \leq 2\gamma
\end{align}

As such, grouping does not change the sensitivity. It does, however, affect the `signal-to-noise` ratio, as explained previously.

\subsection{Details on the Threat Model with/without labels}\label{app:threat_model_formal}

We formalize the threat model with and without labels in this section.
Differential Privacy is defined in Equation~\ref{eqn:definition_dp} and depends on the definition of two adjacent datasets: $\mD,\mD'$: $\mD = \{ \mC_1 \} + \{ \mC_j \}_{j=2}^J$, $\mD' = \{ \mC'_1 \} + \{ \mC_j \}_{j=2}^J$, where without loss of generality, we assume that $\mC_1$ is the sample that changes.

\textbf{Full context privacy } In our method and~\citet{wu2023privacyhistogramllm}, privacy is defined against all in-context examples. This means that $\mD = \{ \mC_j \}_{j=1}^J$ is over $J$ samples in the context. Then, $\mD'$ can have any example in the context changed, like $\mC_1$ in the above example.

\textbf{Assume labels are public } The above is in contrast to the privacy guarantee in~\citet{tang2024dpfewshot}. They assume that labels are known and subsample per label. This means that $\mD = \{ \mC_j \}_{j=1}^J$ are all the $J$ samples from a particular label. Then, $\mD'$ can have any example \textit{for that particular label} changed. This is a less strict definition of privacy because, for $K$ labels, between one dataset and an adjacent dataset, $K$ samples can differ -- one per label.

\begin{figure}[t]
\centering
\begin{minipage}[t]{0.48\textwidth}
\centering
\includegraphics[width=0.99\linewidth]{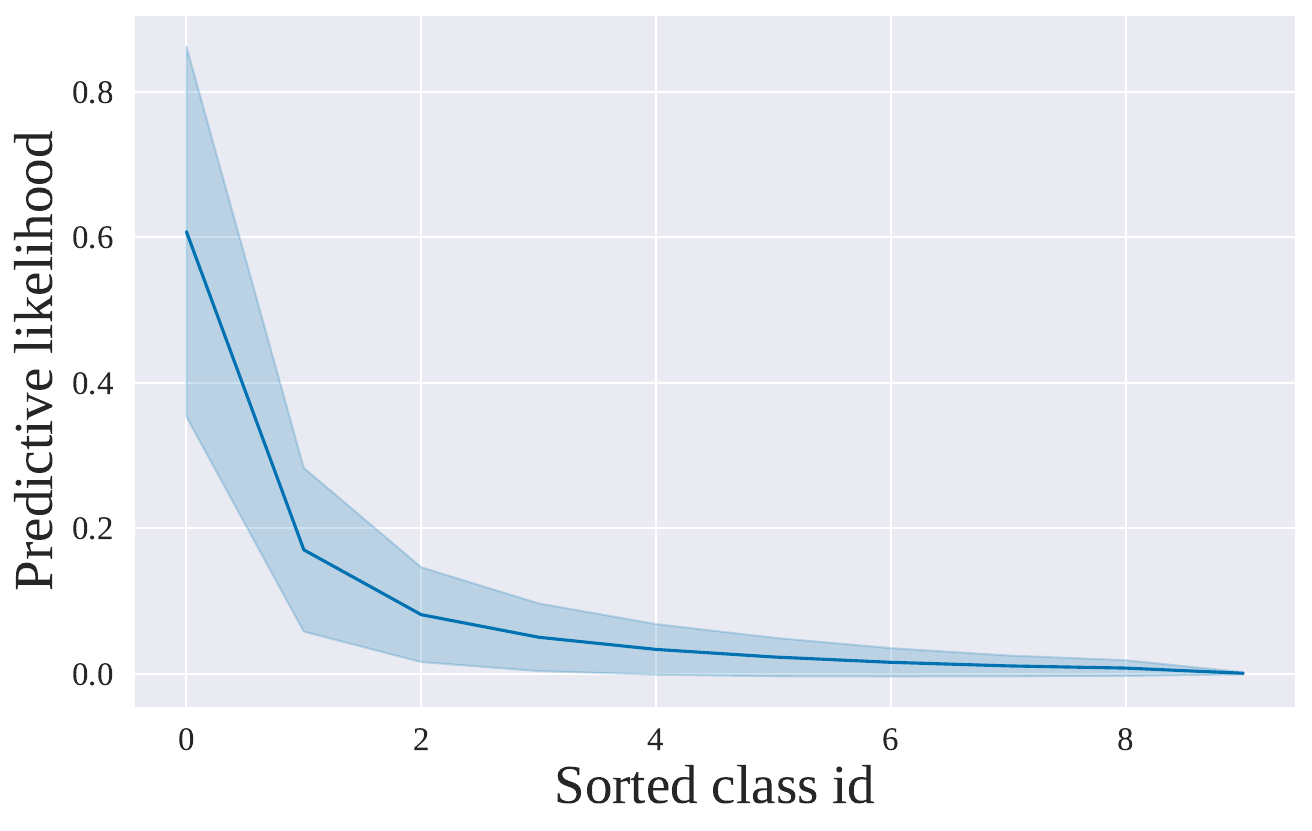}
\captionof{figure}{Predictive likelihoods on a 10-way classification task (GSM8k) with \qwen. The predictions are sorted, and the mean and std. deviation among 3000 random samples are plotted.}
\label{fig:statdiv}
\end{minipage}
\hfill
\begin{minipage}[t]{0.48\textwidth}
\centering
\includegraphics[width=0.99\linewidth]{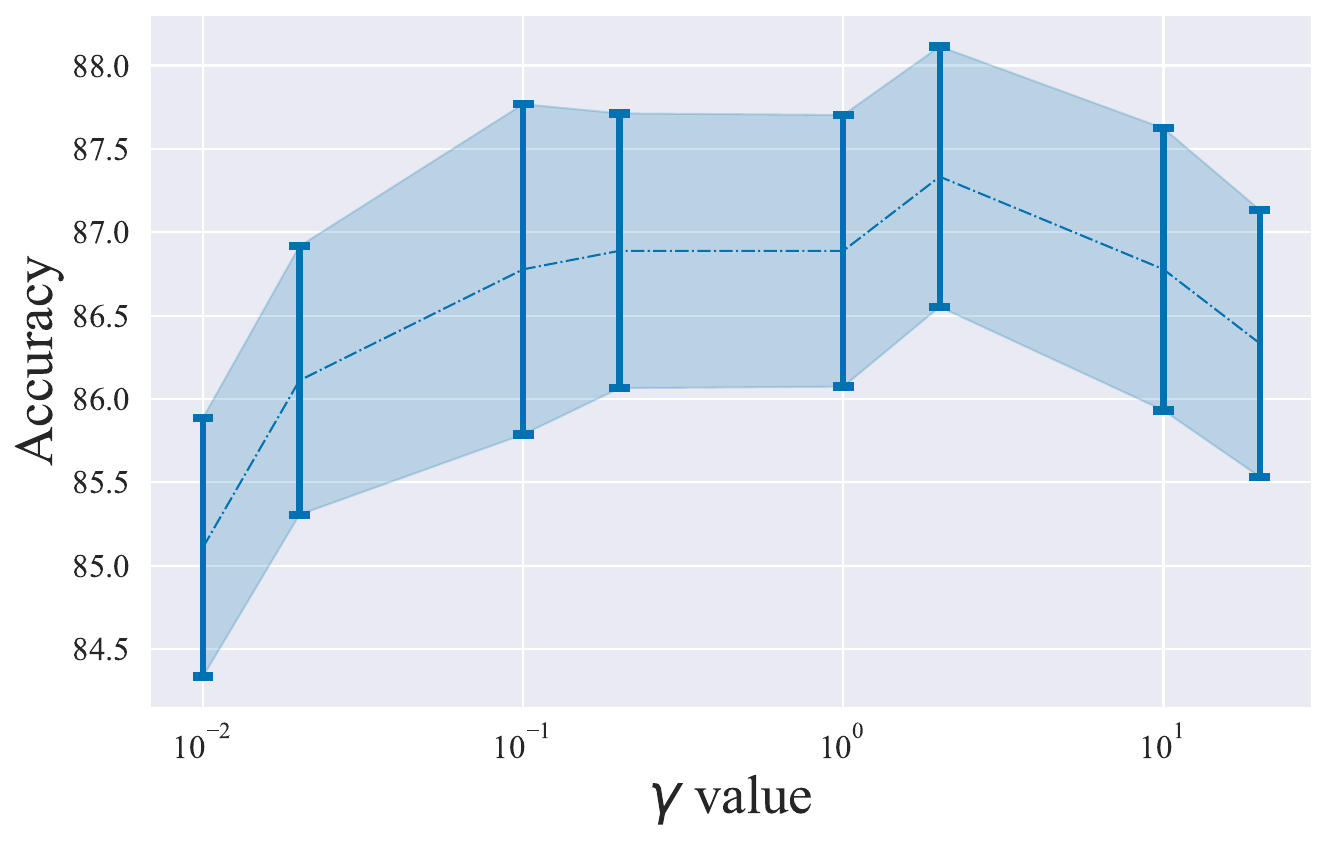}
\captionof{figure}{A hyperparameter sweep for $\gamma$, with \qwen~and 8 in-context examples. The mean and standard error of 25 seeds are plotted. We subsequently have used $\gamma=2$ for all experiments.}
\label{fig:gammasweep}
\end{minipage}
\end{figure}

\section{Comparison of Hard or Soft voting}\label{app:statdiv}

Central to our method is the change from the `Hard Voting' of RNM to the Soft Predictions enabled by our Product of Experts. Subsection~\ref{sec:hard_or_soft} gives an argument that the Kullback-Leibler divergence and $\ell_\infty$ norm are lower for PoE. This section provides more details on this argument. To repeat, a prediction over $K$ classes would be $\vp = [p_1, p_2, \cdots, p_K] \in [0,1]^K$. For hard voting, this would be thresholded to $\vh \in \{0,1\}; \sum \vh = 1$. For soft predictions, this would be clipped to $\vs  = [s_1, s_2, \cdots, s_K] \in [e^{-2},1]^K$.

We compare with the mean predictive likelihood $l_{\text{mean}}(\vq, y) = \vq_{y}$ for approximate vector $\vq$ at label $y$, and the $\ell_\infty$-norm $D_{\infty}(\vp, \vq) = \max_i |p_i - q_i|$. The mean predictive likelihood is important, as it is the probability that the correct will be sampled according to this expert alone; the $\ell_\infty$-norm indicates the largest change in sampling probability between the unclipped and clipped vector. Most predictions from the LLM follow a power-law distribution~\citep{mandelbrot1954structure}. This is empirically verified in Figure~\ref{fig:statdiv}, which shows the mean and standard deviation of the predictive likelihood across 3000 predictions with a Qwen3 model.

As such, we can clarify why the KL divergence and $\ell_\infty$-norm are better for soft predictions. assume w.l.o.g. that the largest predictive likelihood is $p_1$ and the second largest $p_2$, which are, according to Figure~\ref{fig:statdiv} is around \num{0.6} and \num{0.18}. Then the $\ell_\infty$-norm

\begin{equation}
    \ell_{\infty,\text{hard}} =\max_i |p_i - q_i| = \max \{ 1-p_1, p_2 \}\ \ ; \qquad \qquad \ell_{\infty,\text{soft}} \leq e^{-\gamma},
\end{equation}

which, for the power-law like Figure~\ref{fig:statdiv}, $\ell_{\infty,\text{soft}}$ is much smaller than $\ell_{\infty,\text{hard}}$, because $e^{-\gamma} < p_2$. And indeed, experimentally we found that the $\ell_\infty$-norm of soft predictions is around $0.135$, which is close to $\exp[-2]=0.1353$.

\begin{table*}[]
\centering
\setlength{\tabcolsep}{4pt}
\renewcommand{\arraystretch}{1.2}
\caption{Reproducing Table~\ref{tab:bigtable} with a \llama~model. For both $8$ and $25$ context set sizes and across three datasets, our method scores significantly higher accuracy. The DP is set to $\varepsilon=4$, consistent with previous implementations. }
\begin{tabularx}{\textwidth}{lrrrclll}
\toprule
 & $\varepsilon$ & \textbf{Num data} & \textbf{Method} & \textbf{Lbl. Priv.} & \textbf{AGNews} & \textbf{DBPedia} & \textbf{TREC} \\
\midrule
No context & 0 & 0 & Empty & \checkmark & $61.4_{ \pm 0.6}$ & $73.7_{ \pm 0.7}$ & $76.5_{ \pm 0.7}$ \\
Synthetic ~\citep{tang2024dpfewshot} & 4 & \num{30000} & Synth. data & \texttimes & $84.9_{\pm 0.9}$ & $88.0_{\pm 0.8}$ & $62.1_{\pm 1.9}$ \\
\midrule
\qquad \textit{8 context examples} & & & & & & & \\
Privacy-by-Sampling ~\citep{wu2023privacyhistogramllm} & 4 & 8 & Subsampling & \checkmark & $79.5_{ \pm 0.8}$ & $65.0_{ \pm 0.7}$ & $71.5_{ \pm 0.9}$ \\
RNM ~\citep{wu2023privacyhistogramllm} & 4 & 8 & Hard voting & \checkmark & $64.2_{ \pm 0.7}$ & $47.3_{ \pm 0.7}$ & $52.7_{ \pm 0.8}$ \\ \rowcolor{LightCyan}
Product-of-Experts (ours) & 4 & 8 & Soft pred. & \checkmark & $\textbf{85.5}_{ \pm 0.6}$ & $\textbf{90.8}_{ \pm 0.5}$ & $\textbf{78.8}_{ \pm 0.6}$ \\
\qquad \textit{25 context examples} & & & & & & & \\
Privacy-by-Sampling ~\citep{wu2023privacyhistogramllm} & 4 & 25 & Subsampling & \checkmark & $80.8_{ \pm 0.8}$ & $71.8_{ \pm 0.7}$ & $72.9_{ \pm 0.8}$ \\
RNM ~\citep{wu2023privacyhistogramllm} & 4 & 25 & Hard voting & \checkmark & $83.6_{ \pm 0.5}$ & $88.2_{ \pm 0.6}$ & $77.8_{ \pm 0.7}$ \\ \rowcolor{LightCyan}
Product-of-Experts (ours) & 4 & 25 & Soft pred. & \checkmark & $\textbf{85.8}_{ \pm 0.5}$ & $\textbf{91.0}_{ \pm 0.5}$ & $\textbf{80.0}_{ \pm 0.5}$ \\
\midrule
\textcolor{gray}{No-DP} & \textcolor{gray}{$\infty$} & \textcolor{gray}{8} & \textcolor{gray}{In-context} & \textcolor{gray}{-} & $\textcolor{gray}{87.8_{ \pm 0.5}}$ & $\textcolor{gray}{91.6_{ \pm 0.5}}$ & $\textcolor{gray}{79.9_{ \pm 1.1}}$ \\
\textcolor{gray}{No-DP} & \textcolor{gray}{$\infty$} & \textcolor{gray}{25} & \textcolor{gray}{In-context} & \textcolor{gray}{-} & $\textcolor{gray}{88.3_{ \pm 0.7}}$ & $\textcolor{gray}{93.4_{ \pm 0.6}}$ & $\textcolor{gray}{80.3_{ \pm 0.8}}$ \\
\bottomrule
\end{tabularx}
\label{tab:repr_bigtable}
\end{table*}

\section{Extra results}\label{app:extra_results}

The LLM that we use for inference plays a central role in our method. Therefore, we repeat all our experiments using another, independently trained frontier LLM.
The main Table~\ref{tab:bigtable} and Table~\ref{tab:gsm8k_acc} experiments are run with a \qwen~model~\citep{qwen3technicalreport}, available from~\url{huggingface.co/Qwen/Qwen3-4B}. These results are reproduced with a \llama~\citep{llama3}, available from~\url{huggingface.co/meta-llama/Llama-3.1-8B}. Those results are shown in Tables~\ref{tab:repr_bigtable} and~\ref{tab:gsm8k_repr}; all conclusions remain unchanged, and the patterns described in the main text hold for this LLM as well. The vision-language experiment in Table~\ref{tab:vlm_accuracy} is run with a QwenVL2.5 model~\citep{QwenVL}, available from~\url{huggingface.co/Qwen/Qwen2.5-VL-7B-Instruct}. That result is reproduced with an InternVL3.5-2B model~\citep{zhu2025internvl3}. The result is shown in Table~\ref{tab:vlm_internvl}; all conclusions remain unchanged. These additional results were obtained with independently trained models, showing the generalizability of our method.

\textbf{Computational efficiency:} Table~\ref{tab:wallclocktime} reports wallclock times for evaluating $25 \times 100$ random seeds on the AGNews dataset with a single A100 GPU. The privacy-by-sampling method~\citep{wu2023privacyhistogramllm} requires $23.8$ hours for 25 context examples, while the synthetic data approach~\citep{tang2024dpfewshot} takes $6.1$ hours, making both methods computationally prohibitive for practical deployment. In contrast, our Product-of-Experts method achieves wall-clock times similar to RNM ($0.8$ hours for 25 context examples), yet, as shown in Table~\ref{tab:bigtable}, it consistently achieves higher accuracy across all datasets and number of examples. This demonstrates that our approach provides the best privacy-utility-efficiency trade-off among existing differentially private ICL methods.

\begin{figure}[t]
\centering
\begin{minipage}[t]{0.48\textwidth}\centering
\setlength{\tabcolsep}{3pt}
\renewcommand{\arraystretch}{1.3}
\captionof{table}{Reproduction of Table~\ref{tab:gsm8k_acc} with a \llama~model. The baseline accuracy for a model without context is 20\%. Our method, PoE, scores significantly higher accuracy than related work.}
\begin{tabular}{lr r r}
\toprule
& \multicolumn{3}{l}{\textbf{Number of examples}} \\
 & \textbf{4} & \textbf{8} & \textbf{20} \\
\midrule
\textcolor{gray}{Non-private, no DP} & \textcolor{gray}{$37.7_{ \pm 0.7}$} & \textcolor{gray}{$37.6_{ \pm 0.7}$} & \textcolor{gray}{$38.3_{ \pm 0.8}$} \\
\midrule
RNM~\citep{wu2023privacyhistogramllm} & $15.9_{ \pm 0.5}$ & $20.2_{ \pm 0.8}$ & $34.8_{ \pm 0.7}$ \\ \rowcolor{LightCyan}
PoE (Ours) & $\textbf{31.0}_{ \pm 1.0}$ & $\textbf{34.5}_{ \pm 0.8}$ & $\textbf{36.8}_{ \pm 0.9}$ \\
\bottomrule
\end{tabular}
\label{tab:gsm8k_repr}
\end{minipage}
\hfill
\begin{minipage}[t]{0.48\textwidth}
\centering
\captionof{table}{Reproduction of Table~\ref{tab:vlm_accuracy} with an InternVL3.5 model.  The baseline accuracy for a model without context is 14\%.  Our method, PoE, scores significantly higher accuracy than related work.}
\setlength{\tabcolsep}{3pt}
\renewcommand{\arraystretch}{1.3}
\begin{tabular}{lr r r}
\toprule
& \multicolumn{3}{l}{\textbf{Number of examples}} \\
& \textbf{4} & \textbf{8} & \textbf{20} \\
\midrule
\textcolor{gray}{Non-private, no DP} & \textcolor{gray}{$97.1_{\pm 0.1}$} & \textcolor{gray}{$97.6_{\pm 0.1}$} & \textcolor{gray}{$98.2_{\pm 0.1}$} \\
\midrule
RNM~\citep{wu2023privacyhistogramllm} & $27.6_{\pm 0.4}$ & $35.6_{\pm 0.5}$ & $64.9_{\pm 0.5}$ \\ \rowcolor{LightCyan}
PoE (Ours) & $\textbf{67.8}_{\pm 0.4}$ & $\textbf{88.5}_{\pm 0.2}$ & $\textbf{96.5}_{\pm 0.1}$ \\
\bottomrule
\end{tabular}
\label{tab:vlm_internvl}
\end{minipage}
\end{figure}

\section{Details on the experimental settings}\label{app:exp-details-general}

\textbf{Clarification of the random seed generation} Each number that has an accuracy $\pm$standard error is the mean and standard error of 25 randomly rerun experiments. Each experiment evaluates 100 randomly drawn prompts from the test set against a randomly drawn context set. This setting follows the setup of ~\citet{tang2024dpfewshot}. The reason that one does not simply take 2500 random prompts is computational cost. Once the random context set is drawn, preprocessed, and loaded into memory, one can evaluate accuracy across multiple randomly drawn prompts. Because of a different preprocessing library for the Vision-Language model experiments, those experiments report the mean of 2500 randomly sampled context and test sets, along with the Clopper-Pearson confidence interval.

It is important to note that the training images in this context are sampled randomly and are not class-conditional (as was done in previous work, e.g.,~\citet{tang2024dpfewshot}). This random sampling means that sometimes the context will have multiple examples for a class, and sometimes it will have none. For instance, if the classes were A, B, and C. Then, sampling five examples in the context could yield AAABD, indicating that classes C or E lack a demonstration.

\textbf{Hyperparameter settings } Where applicable, the hyperparameters are copied from the previous works that we compare with. Our method introduces only one additional hyperparameter, the clipping bound $\gamma$. We set this to $\gamma = 2$ for all experiments. Figure~\ref{fig:gammasweep} shows a hyperparameter sweep for $\gamma$ on the AGNews dataset with a Qwen3 model. The figure shows the mean and standard error among 25 random seeds. We choose $\gamma=2$, which has the highest mean accuracy. Also, this means that values below $\exp[-2]=13.5$\% are clipped so that the ``uncertainty nuance is preserved'' between \num{13.5} and \num{100}\%.

Similarly, only for the VLM experiments, we `group' shots in sets of 2. This is explained in Appendix~\ref{app:grouping}. For example, for the setting of Table~\ref{tab:vlm_accuracy}, with 8 examplars and our method at $\varepsilon=1.0$, a setting of groupsize 1 would score $50.6_{\pm 1.0}$\% and a groupsize of 4 would score $67.6_{\pm 0.9}$\%, which both have lower accuracy than $77.6_{\pm 0.8}$\%. We hypothesize that VLM predictions are much better across multiple shots than with a single shot. As a thought experiment, when given two images and asked which a third image is most similar to, similarity could be based on foreground/background/color/texture, etc. Whenever two examples each are provided, the VLM can recognise the mode of similarity. As a counterexperiment, we ran AGNews with a group size of 2, but this led to lower accuracy, likely due to the poorer signal-to-noise ratio.

\textbf{Details on the privacy-by-sampling method } The experiments for Table~\ref{tab:bigtable}  compare with the privacy-by-sampling method of ~\citet{wu2023privacyhistogramllm}. Due to computational cost, we use 100 random subsamples, with each sample included with independent probability \num{0.5}. The corresponding privacy accounting is done with the PRV accountant in Opacus~\citep{opacus}.

\textbf{Pretraining and privacy } All LLMs used in this study were pretrained on web-scale data. In some cases, it is difficult to determine if the models were pretrained on the datasets we use for evaluation. Even when some datasets are explicitly filtered out by the original authors of those models and pipelines, snippets of data can still appear in various forms on the web. As such, we design the tasks in this paper such that learning from the actual context is instrumental to achieving high accuracy. In all experiments, we report 0-shot accuracy. In all experiments, the scores achieved by any in-context method are significantly higher than 0-shot performance, providing evidence of in-context learning.

The in-context learning must be inherent to the tasks. Even if a text snippet was inadvertently included in the pretraining data, for example, the LLM still has to learn the class label as prescribed by the context. This is especially so for the Vision-Language task. The pseudo-names are chosen such that it should be learnt from the context which image goes with which pseudo-label~\citep{ICL_VLM_pseudo-name,derakhshani2023self_secat}. The results for those experiments are averaged among 2500 random seeds, where, for each seed, a different image class is drawn and assigned a different pseudo-name.

For some use cases, privacy remains important in the threat model (c.f., Subsection~\ref{sec:scenarios}) even though the data may have been used for pretraining. For example, an image of a particular piece of furniture on the web could be part of the pretraining. However, the fact that the furniture is in someone's living room is private information that should not be leaked by a cleaning robot going from house to house. As a second example, a particular poem might be included in the pretraining. However, the fact that the poem is discussed with a friend via email is private information that an LLM email agent should not leak.

\begin{table}[t]
\centering
\setlength{\tabcolsep}{3pt}
\renewcommand{\arraystretch}{1.3}
\caption{Template prompts for our tasks. The curly brackets, $\{\}$, are replaced with the actual data, and the template is repeated for each expert. The prompt templates for AGNews, DBPedia, and TREC equal those in~\citet{zhao2021calibrate,tang2024dpfewshot}.}
\begin{tabularx}{\linewidth}{lXX}
\toprule
\textbf{Dataset} & \textbf{Template} & \textbf{Labels} \\
\midrule
AGNews & Classify the news articles. Article: \{text\} Answer: \{label\} & World, Sports, Business, Technology \\
DBPedia & Classify the documents based on what they are about. Article: \{text\} Answer type: \{label\} & Number, Location, Person, Description, Entity, Abbreviation \\
TREC & Classify the questions based on their answer type. Question: \{text\} Answer Type: \{label\} & Number, Location, Person, Description, Entity, Abbreviation \\
GSM8k & Answer the grade school math problem with a number. Question: \{text\} Answer: \{label\} & 0, 1, 2, 3, 4, 5, 6, 7, 8, 9 \\
VLM & You are a helpful assistant. Your task is to classify images into one of these categories: \{labels\}. Based on the support examples provided, respond with the single-word label that best describes the query image. \{image\},\{label\} & Dax, Blicket, Perpo, Slation, Shously \\
\bottomrule
\end{tabularx}
\label{tab:prompttemplate}
\end{table}
\subsection{Details on the Vision Language Model eval}\label{app:exp-details-vlm}

\begin{figure}[t!]
    \centering
    \includegraphics[width=\linewidth]{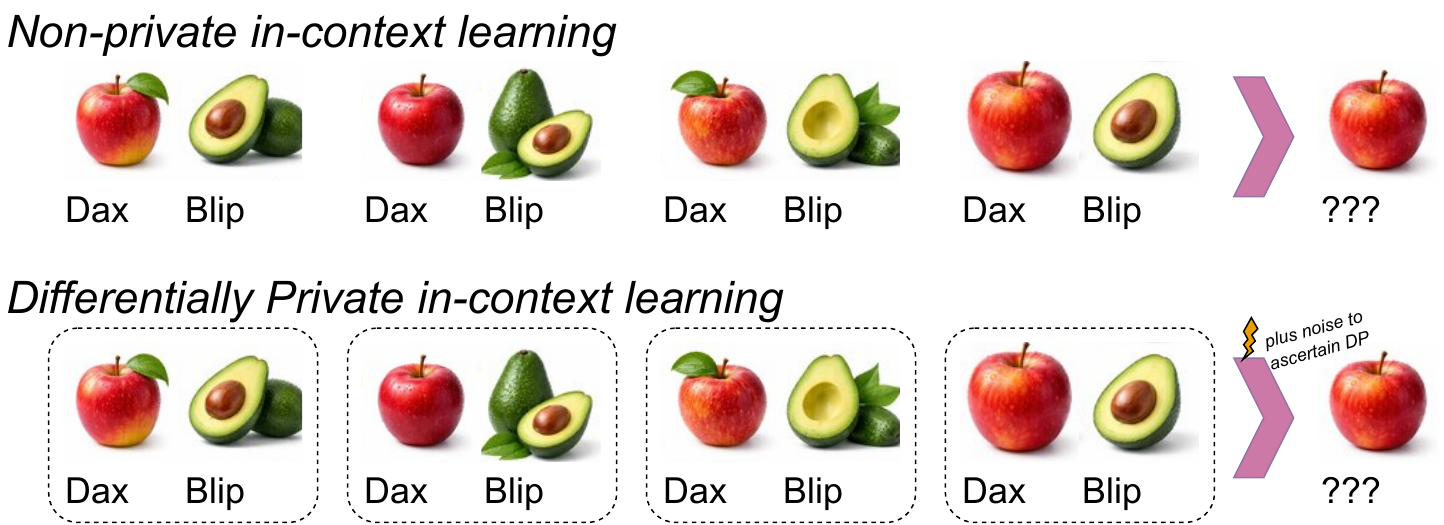}
    \caption{Evaluation setting for in-context learning with a Vision-Language Model (VLM). The names of the classes are deliberately chosen to be nonsensical words, forcing the VLM to learn from context rather than rely on knowledge from pretraining.}
    \label{fig:vlm_eval_setup}
\end{figure}

We visualize the experimental setup for the VLM eval in Figure~\ref{fig:vlm_eval_setup}.
The pseudo-name labelling tasks require the VLM to learn to label the object's name in the image. This is a challenging task for the VLM, since it must learn to reason about the object and its relationships to other objects in the image. The pseudo-names are five English-sounding but non-real words following previous work~\citep{ICL_VLM_pseudo-name}. This choice prevents the VLM from relying on its pretraining knowledge to label objects. The names are ``perpo'', ``blicket'', ``dax'', ``slation'', and ``shously''. The prompt for this task and other tasks are outlined in Table~\ref{tab:prompttemplate}.

Images for this task are filled from the ImageNet dataset~\citep{deng2009imagenet}. The evaluation relies on 2500 random seeds. That means that for every random seed, a randomly chosen ImageNet class is associated with each of the five pseudo-names. The number of context examples varies and is mentioned in each table. We group the shots into pairs using the grouping described in Section~\ref{app:grouping}. This does not affect the privacy guarantee as explained in that section, but it significantly improves the results. We hypothesize that this is because, given two images per class, the VLM can reason more effectively about the foreground versus the background.
It is important to note that, for the VLM and other tasks using our method, the training images are sampled randomly rather than conditionally on class. This is explained in subsection~\ref{app:exp-details-general}.

\begin{table}[t!]
    \centering
    \captionof{table}{Runtimes for the three methods. This is the wall-clock time for $25 \times 100$ random seeds on a single A100.}
    \setlength{\tabcolsep}{3pt}
    \renewcommand{\arraystretch}{1.3}
    \begin{tabular}{l c c c}
    \toprule
    & Privacy-by-sampling & Synthetic data as bottleneck & RNM / PoE\\
    & \citep{wu2023privacyhistogramllm} & \citep{tang2024dpfewshot} & Wu et al. / ours \\
    \midrule
25 context experiments & $23.8$ & $6.1$  & $0.8$  \\
8 context experiments & $8.3$  & $6.2$   & $0.4$   \\
    \bottomrule
    \end{tabular}
    \label{tab:wallclocktime}
\end{table}

\subsection{Details on the Membership Inference Attack}\label{app:exp-details-mia}

The MIA serves to highlight the empirical privacy vulnerability of in-context learning and to establish the effect of the DP guarantee. We use a Likelihood Ratio attack, which is common in the literature~\citep{shokri2017membership,carlini_ref_model_MIA,duan2024MIAonICL}. In this case, the likelihood is the predictive score that the LLM assigns to the classification token. Note that this is a simulated version of in-context learning, where only the class prediction is output, not its logits. Although label-only MIAs have been explored, we demonstrate the vulnerability of the likelihood-ratio variant, as this is a precursor to label-only attacks~\citep{he2025labelonlyMIA}. We do not use a reference model, since we assume the task is novel to the LLM and thus the reference would be uniform. This setting and assumptions follow the situation in~\citet{duan2024MIAonICL}.

The attack scenario is as follows: the context is loaded with data from $N$ private data sources; the attacker can input any query and will receive a predicted class and its logits. The attacker uses this logit to decide if the query is part of the private data sources. The attacker is successful if the query is part of the private data sources.

The success metric is the Area Under the Receiver Operating Characteristic Curve (AUROC) as a function of the threshold $\tau$ that separates positive and negative samples. This AUROC is based on the True Positive Rate (TPR) and the False Positive Rate (FPR):

\begin{itemize}[leftmargin=1em]
    \item A True Positive (TP) is a query that is part of the private data sources and is correctly classified as such.
    \item A False Positive (FP) is a query that is \textbf{not} part of the private data sources, but is incorrectly classified as such.
\end{itemize}

We run the MIA on three text classification datasets to demonstrate generalizability across datasets. Figure~\ref{fig:mia_auroc} shows the AUROC for the four text classification datasets. The results were obtained with 50 random seeds and 20 in-context examples each. This means the AUROC is determined from member data across 1000 attacks. The model indicated as DP was used with $\varepsilon = 1$, and our method, described in Algorithm \ref{alg:dp_icl_ci}, was used.

Across the three figures, it can be seen that our method significantly improves AUROC, bringing it closer to \num{0.5}, the level of random guessing. For example, on the AGNews dataset, our method brings the AUROC from \num{60.0} to \num{52.9}, and on the TREC dataset, it brings the AUROC from \num{63.6} to \num{53.8}. This shows that a) in-context learning is vulnerable to membership inference attacks and b) our method is effective at reducing this vulnerability.

\begin{figure*}[htbp]
    \centering
    \includegraphics[width=\textwidth]{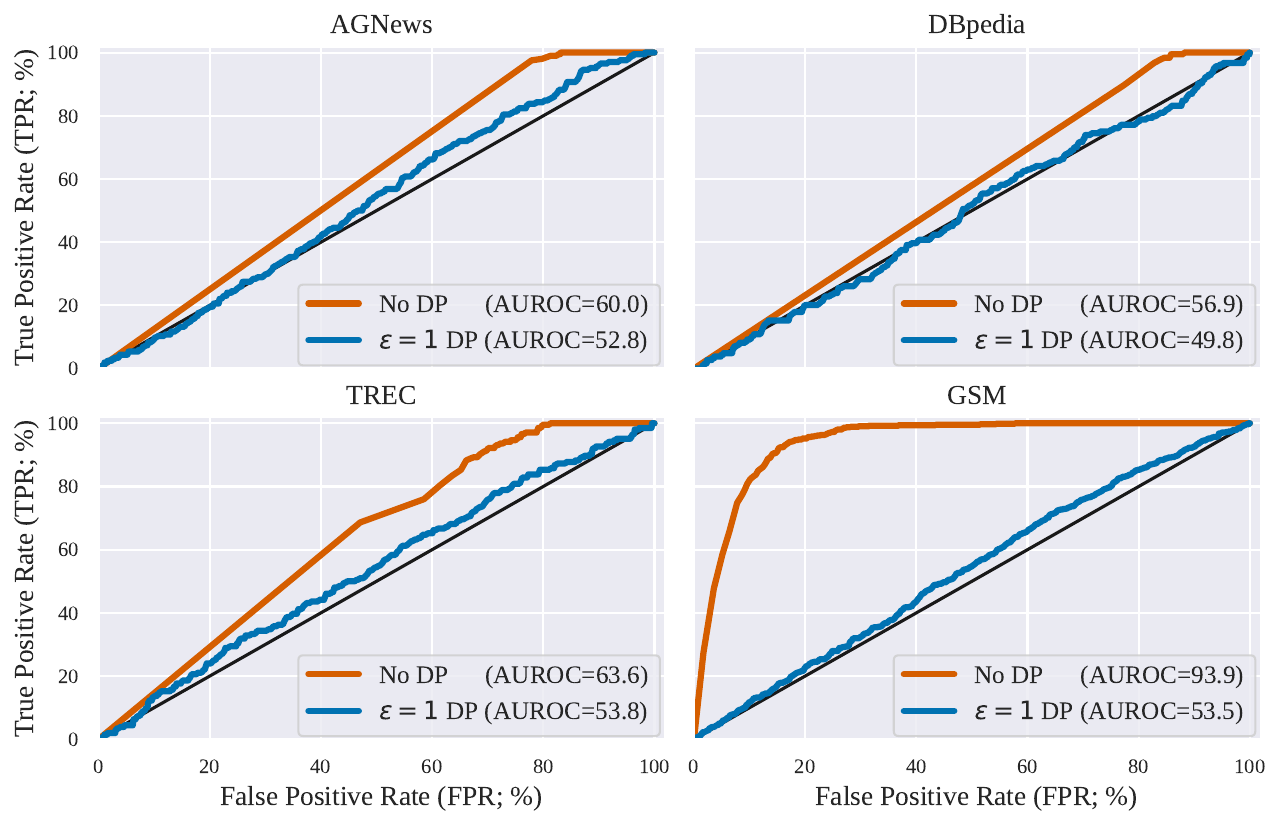}
    \caption{Membership Inference Attack (MIA) on the AGNews, DBPedia, TREC, and GSM8k datasets. In each MIA, the ROC curve for our DP method approaches the random chance line, with an Area Under the ROC curve of closer to 50, indicating lower empirical vulnerability to the privacy attack. The difference is most visible in the GSM8k dataset, where longer contexts are used.}
    \label{fig:mia_auroc}
\end{figure*}

\end{document}

%% file: math_commands.tex
\usepackage{amsmath,amsfonts,bm}
\usepackage{cancel}

\def\eqref#1{eq.~\ref{#1}}

\def\1{\bm{1}}

\def\vh{{\bm{h}}}

\def\vl{{\bm{l}}}
\def\vm{{\bm{m}}}

\def\vp{{\bm{p}}}
\def\vq{{\bm{q}}}

\def\vs{{\bm{s}}}

\def\vx{{\bm{x}}}
\def\vy{{\bm{y}}}
\def\vz{{\bm{z}}}

\def\mC{{\bm{C}}}
\def\mD{{\bm{D}}}

\def\mS{{\bm{S}}}

\DeclareMathAlphabet{\mathsfit}{\encodingdefault}{\sfdefault}{m}{sl}
\SetMathAlphabet{\mathsfit}{bold}{\encodingdefault}{\sfdefault}{bx}{n}

\newcommand*{\prob}[1]{\mathbb{P}}

\newcommand{\lTwo}{\ell_2}

\DeclareMathOperator*{\argmax}{arg\,max}

\usepackage{centernot}
